\documentclass[a4paper,reqno,gray]{amsart}
\usepackage{cmap} 
\usepackage[british]{babel}
\usepackage[T1]{fontenc}
\usepackage{mathtools}
  \mathtoolsset{mathic}
  \allowdisplaybreaks
\usepackage{textcomp}
\usepackage[spacing]{microtype}
\usepackage{url}
\usepackage{tikz}
\usepackage{enumerate}
\usepackage{paralist}
\usepackage{amsmath,amssymb,amsfonts} 
\usepackage{amsthm}
\usepackage{mathtools}
\usepackage{mathptmx}
\usepackage{helvet}
\usepackage{courier}
\usepackage{microtype}
\usepackage{hyperref,url}
  \hypersetup{  
    bookmarksnumbered
  }
  \hypersetup{  
    breaklinks=false,
    pdfborderstyle={/S/U/W 0},  
    citebordercolor=.235 .702 .443,
    urlbordercolor=.255 .412 .882,
    linkbordercolor=.804 .149 .19,
  }
  \hypersetup{ 
    pdfauthor={Gert de Cooman, Filip Hermans, Alessandro Antonucci, and Marco Zaffalon},
    pdftitle={Epistemic irrelevance in credal nets: the case of imprecise Markov trees},
    pdfkeywords={coherence, credal net, epistemic irrelevance, epistemic independence, strong independence, imprecise Markov tree, separation, hidden Markov model}
  }
\usepackage[all]{hypcap} 
\usepackage{fixltx2e} 
\usepackage{wrapfig}
\usepackage{ifthen}

\usepackage{nicefrac}
\usepackage{tikz}
\usetikzlibrary{calc}
\usetikzlibrary{decorations.pathmorphing}
\usetikzlibrary{decorations.markings}
\usetikzlibrary{fit}
\usetikzlibrary{backgrounds}
\usetikzlibrary{trees}
\usetikzlibrary{matrix}

\def\afstand{1.6cm} 

\xdefinecolor{rood}{RGB}{241, 90, 34}
\xdefinecolor{geel}{RGB}{255, 197, 11}
\xdefinecolor{roze}{RGB}{236, 0, 140}
\xdefinecolor{groen}{RGB}{179, 211, 53}
\xdefinecolor{grijs}{RGB}{72, 119, 116}

\newcommand{\printMessage}[3]{
\ifthenelse{\equal{#3}{left}}{
  \draw[message] ($(#1)!.30!(#2) !.2!90:(#1)$) -- ($(#1)!.70!(#2) !-.2!90:(#2)$) node[midway] (m) {};
}
{
  \draw[message] ($(#1)!.30!(#2) !.2!-90:(#1)$) -- ($(#1)!.70!(#2) !-.2!-90:(#2)$) node[midway] (m) {};
  
}
}
\tikzstyle{nd}=[circle,draw=black!80,fill=geel!15,minimum size=.65cm,line width=.2pt]
\tikzstyle{observed}=[nd,fill=groen!100]
\tikzstyle{target}=[nd,fill=roze!100]
\tikzstyle{unobserved}=[nd,fill=geel!15]
\tikzstyle{terminal}=[nd,fill=geel!15]
\tikzstyle{ed}=[draw=geel!60!black,line width=.8pt, postaction={decorate}, decoration={markings,mark=at position 1.0 with {\arrow[draw=geel!60!black,line width=.8pt]{>}}}]
\tikzstyle{backbone}=[draw=rood,line width=1.8pt, postaction={decorate}, decoration={markings,mark=at position 1.0 with {\arrow[draw=rood,line width=.8pt]{>}}}]
\tikzstyle{message}=[draw=grijs,line width=.5pt, postaction={decorate}, decoration={markings,mark=at position 1.0 with {\arrow[draw=grijs,line width=.8pt]{>}}}]

\usepackage{wrapfig}
\usepackage{ifthen}
\usepackage{listings}
\newtheorem{definition}{Definition}
\newtheorem{theorem}{Theorem}
\newtheorem{proposition}[theorem]{Proposition}

\newcommand{\inout}[2]{#1\cup{#2}}
\newcommand{\reals}{\mathbb{R}}
\newcommand{\set}[2]{\{#1\colon#2\}}
\newcommand{\abs}[1]{\lvert#1\rvert}
\newcommand{\then}{\Rightarrow}
\newcommand{\ifonlyif}{\Leftrightarrow}
\newcommand{\nodes}{T}
\newcommand{\nonterminals}{\nodes^\lozenge}
\newcommand{\trunk}{\tilde{\nodes}}
\newcommand{\init}{\square}

\newcommand{\mother}[1]{m(#1)}
\newcommand{\children}[2][]{C_{#1}(#2)}
\newcommand{\siblings}[1]{S(#1)}
\newcommand{\ancestors}[1]{A(#1)}

\newcommand{\until}[1]{{\uparrow}#1}
\newcommand{\descendants}[1]{D(#1)}
\newcommand{\after}[1]{{\downarrow}#1}
\newcommand{\precedes}{\sqsubseteq}
\newcommand{\sprecedes}{\sqsubset}
\newcommand{\nprecedes}{\not\sqsubseteq}

\newcommand{\nonparnondes}[1]{\overline{#1}}

\newcommand{\var}[1]{X_{#1}}
\newcommand{\values}[1]{\mathcal{X}_{#1}}
\newcommand{\valuesio}[2]{\mathcal{X}_{\inout{#1}{#2}}}
\newcommand{\gambles}[1]{\mathcal{L}(\values{#1})}
\newcommand{\gamblesio}[2]{\mathcal{L}(\valuesio{#1}{#2})}
\newcommand{\xval}[1]{x_{#1}}
\newcommand{\zval}[1]{z_{#1}}
\newcommand{\pr}{P}
\newcommand{\lpr}{\underline{\pr}}
\newcommand{\upr}{\overline{\pr}}
\newcommand{\lupr}{\overline{\underline{\pr}}}
\newcommand{\apr}{Q}
\newcommand{\alpr}{\underline{\apr}}
\newcommand{\aupr}{\overline{\apr}}
\newcommand{\alupr}{\underline{\overline{\apr}}}
\newcommand{\rpr}{R}
\newcommand{\rlpr}{\underline{\rpr}}
\newcommand{\rupr}{\overline{\rpr}}
\newcommand{\dpr}{V}
\newcommand{\dlpr}{\underline{\dpr}}
\newcommand{\dupr}{\overline{\dpr}}
\newcommand{\dlcon}[3][\cdot]{\dlpr_{#2}(#1\vert\var{#3})}
\newcommand{\ducon}[3][\cdot]{\dupr_{#2}(#1\vert\var{#3})}
\newcommand{\xindlcon}[3][\cdot]{\dlpr_{#2}(#1\vert\xval{#3})}
\newcommand{\xinducon}[3][\cdot]{\dupr_{#2}(#1\vert\xval{#3})}
\newcommand{\llocconm}[2][\cdot]{\alpr_{#2}({#1}\vert\var{\mother{#2}})}
\newcommand{\bigllocconm}[2][\cdot]{\alpr_{#2}\bigg({#1}\Big\vert\var{\mother{#2}}\bigg)}
\newcommand{\lloccon}[3][\cdot]{\alpr_{#2}({#1}\vert\var{#3})}
\newcommand{\loccon}[3][\cdot]{\apr_{#2}({#1}\vert\var{#3})}
\newcommand{\ulocconm}[2][\cdot]{\aupr_{#2}({#1}\vert\var{\mother{#2}})}
\newcommand{\bigulocconm}[2][\cdot]{\aupr_{#2}\bigg({#1}\Big\vert\var{\mother{#2}}\bigg)}
\newcommand{\lulocconm}[2][\cdot]{\alupr_{#2}({#1}\vert\var{\mother{#2}})}
\newcommand{\luloccon}[3][\cdot]{\alupr_{#2}({#1}\vert\var{#3})}
\newcommand{\biglulocconm}[2][\cdot]{\alupr_{#2}\bigg({#1}\Big\vert\var{\mother{#2}}\bigg)}
\newcommand{\lloc}[1]{\alpr_{#1}}
\newcommand{\uloc}[1]{\aupr_{#1}}
\newcommand{\loc}[1]{\apr_{#1}}
\newcommand{\xinlloccon}[3][\cdot]{\alpr_{#2}({#1}\vert\xval{#3})}

\newcommand{\zinllocconm}[2][\cdot]{\alpr_{#2}({#1}\vert\zval{\mother{#2}})}
\newcommand{\lglobcon}[3][\cdot]{\lpr_{\after{#2}}({#1}\vert\var{#3})}
\newcommand{\luglobcon}[3][\cdot]{\lupr_{\after{#2}}({#1}\vert\var{#3})}
\newcommand{\lglobconirr}[4][\cdot]{\lpr_{\after{#2}}({#1}\vert\var{\sing{#3}\cup{#4}})}
\newcommand{\lglobconm}[2][\cdot]{\lpr_{\after{#2}}({#1}\vert\var{\mother{#2}})}
\newcommand{\luglobconm}[2][\cdot]{\lupr_{\after{#2}}({#1}\vert\var{\mother{#2}})}

\newcommand{\uglobcon}[3][\cdot]{\upr_{\after{#2}}({#1}\vert\var{#3})}
\newcommand{\uglobconm}[2][\cdot]{\upr_{\after{#2}}({#1}\vert\var{\mother{#2}})}
\newcommand{\lglob}[1]{\lpr_{#1}}
\newcommand{\uglob}[1]{\upr_{#1}}
\newcommand{\dlglobcon}[3][\cdot]{\dlpr_{\after{#2}}({#1}\vert\var{#3})}
\newcommand{\dlglobconirr}[4][\cdot]{\dlpr_{\after{#2}}({#1}\vert\var{\sing{#3}\cup{#4}})}
\newcommand{\dlglobconm}[2][\cdot]{\dlpr_{\after{#2}}({#1}\vert\var{\mother{#2}})}

\newcommand{\dlglob}[1]{\dlpr_{#1}}
\newcommand{\zindlglobcon}[3][\cdot]{\dlpr_{\after{#2}}({#1}\vert\zval{#3})}

\newcommand{\xinlglobcon}[3][\cdot]{\lpr_{\after{#2}}({#1}\vert\xval{#3})}

\newcommand{\xindlglobcon}[3][\cdot]{\dlpr_{\after{#2}}({#1}\vert\xval{#3})}
\newcommand{\xinuglobcon}[3][\cdot]{\upr_{\after{#2}}({#1}\vert\xval{#3})}
\newcommand{\zinlglobcon}[3][\cdot]{\lpr_{\after{#2}}({#1}\vert\zval{#3})}
\newcommand{\zinlglobconirr}[4][\cdot]{\lpr_{\after{#2}}({#1}\vert\zval{\sing{#3}\cup{#4}})}
\newcommand{\zindlglobconirr}[4][\cdot]{\dlpr_{\after{#2}}({#1}\vert\zval{\sing{#3}\cup{#4}})}
\newcommand{\zinuglobcon}[3][\cdot]{\upr_{\after{#2}}({#1}\vert\zval{#3})}

\newcommand{\linecon}[2][\cdot]{\lpr_{\after{\children{#2}}}({#1}\vert\var{#2})}
\newcommand{\biglinecon}[2][\cdot]{\lpr_{\after{\children{#2}}}\bigg({#1}\Big\vert\var{#2}\bigg)}
\newcommand{\uinecon}[2][\cdot]{\upr_{\after{\children{#2}}}({#1}\vert\var{#2})}
\newcommand{\biguinecon}[2][\cdot]{\upr_{\after{\children{#2}}}\bigg({#1}\Big\vert\var{#2}\bigg)}
\newcommand{\luinecon}[2][\cdot]{\lupr_{\after{\children{#2}}}({#1}\vert\var{#2})}
\newcommand{\bigluinecon}[2][\cdot]{\lupr_{\after{\children{#2}}}\bigg({#1}\Big\vert\var{#2}\bigg)}
\newcommand{\dlinecon}[2][\cdot]{\dlpr_{\after{\children{#2}}}({#1}\vert\var{#2})}

\newcommand{\xinlinecon}[2][\cdot]{\lpr_{\after{\children{#2}}}({#1}\vert\xval{#2})}
\newcommand{\xinuinecon}[2][\cdot]{\upr_{\after{\children{#2}}}({#1}\vert\xval{#2})}

\newcommand{\xindlinecon}[2][\cdot]{\dlpr_{\after{\children{#2}}}({#1}\vert\xval{#2})}

\newcommand{\lupcon}[3][\cdot]{\rlpr_{#2}({#1}\vert\xval{#3})}
\newcommand{\uupcon}[3][\cdot]{\rupr_{#2}({#1}\vert\xval{#3})}
\newcommand{\xindic}[1]{\ind{\{\xval{#1}\}}}
\newcommand{\zindic}[1]{\ind{\{\zval{#1}\}}}
\newcommand{\clpr}[3][\cdot]{\underline{P}_{#2}({#1}\vert\var{#3})}
\newcommand{\mlpr}[1]{\underline{P}_{#1}}
\newcommand{\mupr}[1]{\overline{P}_{#1}}

\newcommand{\zinclpr}[3][\cdot]{\underline{P}_{#2}({#1}\vert\zval{#3})}

\newcommand{\mass}{q}
\newcommand{\lmass}{\underline{\mass}}
\newcommand{\umass}{\overline{\mass}}
\newcommand{\lumass}{\underline{\overline{\mass}}}
\newcommand{\masscon}[2][\cdot]{\mass({#1}\vert\var{#2})}

\newcommand{\target}{r}
\newcommand{\ltarget}{\underline{\target}}
\newcommand{\utarget}{\overline{\target}}
\newcommand{\family}[1]{\mathcal{I}(\mlpr{#1})}
\newcommand{\familytoo}[1]{\mathcal{I}(\mlprtoo{#1})}
\newcommand{\treefamily}[1]{\mathcal{T}(\underline{#1})}
\newcommand{\sing}[1]{\{#1\}}
\newcommand{\xsing}[1]{\sing{\xval{#1}}}
\newcommand{\zsing}[1]{\sing{\zval{#1}}}
\newcommand{\msing}[1]{\sing{\mother{#1}}}
\newcommand{\ind}[1]{\mathbb{I}_{#1}}
\newcommand{\xinlrext}[2][\cdot]{\underline{R}({#1}\vert\xval{#2})}

\newcommand{\zinlrext}[2][\cdot]{\underline{R}({#1}\vert\zval{#2})}

\newcommand{\lrext}[2][\cdot]{\underline{R}({#1}\vert\var{#2})}

\newcommand{\msg}{\pi}
\newcommand{\varmsg}{\varpi}
\newcommand{\Msg}{\Pi}
\newcommand{\amsg}{\lambda}
\newcommand{\lmess}[1]{{\underline{\msg}}_{#1}}
\newcommand{\lvarmess}[1]{{\underline{\varmsg}}_{#1}}
\newcommand{\lMess}[1]{{\underline{\Msg}}_{#1}}
\newcommand{\umess}[1]{{\overline{\msg}}_{#1}}
\newcommand{\uvarmess}[1]{{\overline{\varmsg}}_{#1}}
\newcommand{\uMess}[1]{{\overline{\Msg}}_{#1}}
\newcommand{\uamess}[1]{{\overline{\amsg}}_{#1}}
\newcommand{\lumess}[1]{{\underline{\overline{\msg}}}_{#1}}
\newcommand{\luvarmess}[1]{{\underline{\overline{\varmsg}}}_{#1}}
\newcommand{\luMess}[1]{{\underline{\overline{\Msg}}}_{#1}}
\newcommand{\mess}[2]{\msg_{#1}^{#2}}
\newcommand{\varmess}[2]{\varmsg_{#1}^{#2}}
\newcommand{\sprod}[2]{\underline{\overline{#1}}\langle#2\rangle}

\newcommand{\toovalues}{\mathcal{Y}}
\newcommand{\vartoo}{Y}
\newcommand{\valtoo}{y}
\newcommand{\valuestoo}[1]{\mathcal{X}_{#1}\times\toovalues}

\newcommand{\gamblesiotoo}[2]{\mathcal{L}(\valuestoo{\inout{#1}{#2}})}
\newcommand{\clprtoo}[3][\cdot]{\underline{P}_{#2}({#1}\vert\var{#3},\vartoo)}

\newcommand{\zinclprtoo}[3][\cdot]{\underline{P}_{#2}({#1}\vert\zval{#3},\valtoo)}
\newcommand{\mlprtoo}[2][\cdot]{\underline{P}_{#2}(#1\vert\vartoo)}
\newcommand{\inmlprtoo}[2][\cdot]{\underline{P}_{#2}(#1\vert\valtoo)}

\DeclareMathOperator{\argmax}{argmax}

\title{Epistemic irrelevance in credal nets: the case of imprecise Markov trees}
\author{Gert de Cooman}
\author{Filip Hermans}
\email{\{gert.decooman,filip.hermans\}@UGent.be}
\address{Ghent University, SYSTeMS Research Group, Technologiepark 914,  9052 Zwijnaarde, Belgium.}
\author{Alessandro Antonucci}
\author{Marco Zaffalon}
\email{\{alessandro,zaffalon\}@idsia.ch}
\address{IDSIA, Galleria 2, 6928 Manno (Lugano), Switzerland}
\begin{document}

\begin{abstract}
  We focus on credal nets, which are graphical models that generalise Bayesian nets to imprecise probability. 
  We replace the notion of strong independence commonly used in credal nets with the weaker notion of epistemic irrelevance, which is arguably more suited for a behavioural theory of probability. 
  Focusing on directed trees, we show how to combine the given local uncertainty models in the nodes of the graph into a global model, and we use this to construct and justify an exact message-passing algorithm that computes updated beliefs for a variable in the tree. 
  The algorithm, which is linear in the number of nodes, is formulated entirely in terms of coherent lower previsions, and is shown to satisfy a number of rationality requirements. 
  We supply examples of the algorithm's operation, and report an application to on-line character recognition that illustrates the advantages of our approach for prediction. 
  We comment on the perspectives, opened by the availability, for the first time, of a truly efficient algorithm based on epistemic irrelevance.
\end{abstract}


\maketitle

\section{Introduction}
\label{sec:introduction}
The last twenty years have witnessed a rapid growth of \emph{graphical models} in the fields of artificial intelligence and statistics. 
These models combine graphs and probability to address complex multivariate problems in a variety of domains, such as medicine, finance, risk analysis, defence, and environment, to name just a few.
\par
Much has been done also on the front of imprecise probability. 
In particular, \emph{credal nets} \cite{cozman2000}  have been and still are the subject of intense research. 
A credal net creates a global model of a domain by combining local uncertainty models using some notion of independence, and then uses this to do inference. 
The local models represent uncertainty by closed convex sets of probabilities, also called \emph{credal sets}.
\par
The notion of independence used with credal nets in the vast majority of cases is that of \emph{strong independence} (with some exceptions in \cite{campos2007b}). 
Loosely speaking, two variables $X,Y$ are strongly independent if the credal set for $(X,Y)$ can be regarded as originating from a number of precise models in each of which $X$ and $Y$ are stochastically independent. 
Strong independence is closely related to the \emph{sensitivity analysis} interpretation of credal sets, which regards an imprecise model as arising out of partial ignorance of a precise one. 
\par
In the particular case of credal nets, strong independence leads to a mathematical equivalence: a credal net model is equivalent to a model consisting of a set of Bayesian nets, each with the same graph but with different values for the parameters. 
The sensitivity analysis interpretation is then that there is some (kind of ideal) Bayesian net model of the problem under consideration, and the graph of such a net is known. 
But, for some reason, the net's parameters are not known precisely, and that is why one considers the set of all the Bayesian nets that are consistent with the partial specification of the parameters. 
Common causes for the existence of partial knowledge are the cost of, and time constraints on, eliciting parameters, and disagreement amongst a group of experts consulted for that purpose. 
Non-ignorable missing data can be another reason, in case the parameters are inferred from a data set \cite{zaffalon2009}.
\par
The sensitivity analysis interpretation of imprecise-probability models, and hence strong independence, is not always applicable. 
A notable case arises when one wishes to model an expert's beliefs: it is then not always tenable that there should be some ideal Bayesian net that models these beliefs, and that it is only because of our limited resources that we cannot define it precisely. 
Rather, it seems more reasonable to concede that expert knowledge may be \emph{inherently} imprecise to some extent.\footnote{For a detailed argumentation and exposition of this point of view, we refer to \cite[Chapter~5]{walley1991}.} 
This simple observation makes the sensitivity analysis interpretation fail, and hence it makes strong independence an inadequate model, in general, for such a situation.\footnote{Obviously, there will be special cases where strong independence is justified in order to model an expert's knowledge. Moreover, strong independence could provide a good approximation to more accurate models, even when it is not entirely appropriate. This is something that seems to deserve further investigation.}
\par
An alternative and attractive approach to expressing irrelevance that is not committed to the sensitivity analysis interpretation is offered by \emph{epistemic irrelevance} \cite{walley1991}: we say that $X$ is epistemically irrelevant to $Y$ if observing $X$ does not affect our beliefs about $Y$. 
In other words, by making an epistemic irrelevance assessment, a subject states that her belief model about $Y$ does (or will) not change after receiving information about $X$. 
When the belief model is a precise probability, both epistemic irrelevance and strong independence reduce to the usual (stochastic) independence.\footnote{If we ignore issues related to events with probability zero.} 
But when the model is a set of probabilities, this is no longer the case, because in contradistinction with strong independence, epistemic irrelevance is a property \emph{of this set} that cannot be explained using properties of the precise probabilities in the set. 
Epistemic irrelevance is defined directly in terms of a subject's belief model (the set of probabilities). 
For this reason, it is very well suited for a behavioural theory of imprecise probability. 
Contrary to strong independence, it is not a symmetrical notion: generally speaking, the epistemic irrelevance of $X$ to $Y$ does not entail the epistemic irrelevance of $Y$ to $X$.
It is also weaker than strong independence, in the sense that strong independence implies epistemic irrelevance: sets of probabilities that correspond to assessments of epistemic irrelevance usually include those related to strong independence assessments. 
It therefore does not lead to overconfident inferences when the sensitivity analysis interpretation is not justified.
\par
At this point, the question we address in this paper should be clear: can we define credal nets based on epistemic irrelevance, and moreover create an exact algorithm to perform efficient inferences with them? 
We give a fully positive answer to this question in the special case that
\begin{inparaenum}[(i)]
\item the graph under consideration is a directed tree, and
\item the related variables assume finitely many values.
\end{inparaenum}
The intuitions that showed us the way towards this result originated in previous work done by some of us on imprecise probability trees \cite{cooman2007d} and imprecise Markov chains \cite{cooman2008}. 
\par
How do we address this problem? 
\par
In Section~\ref{sec:markov-trees}, we discuss some preliminary graph-theoretic notions, and define the local uncertainty models that will be used at each node of a tree. 
These models are formalised through the language of \emph{coherent lower previsions} \cite{walley1991}. 
We discuss how such local models will give rise to a global uncertainty model, which plays the same role as the joint mass function built by the chain rule in a Bayesian net. 
Based on the global model, we state the Markov condition that defines the imprecise-probability interpretation of our credal trees.
As announced before, this Markov condition involves epistemic irrelevance rather than strong independence. 
\par
In Section~\ref{sec:independent-natural-extension}, we take a brief detour to discuss in general terms how to combine marginal models into joint ones using irrelevance assessments, in a way that is as conservative as possible. 
We do so because the notion of so-called \emph{epistemic independence}, which arises out of a symmetrisation of epistemic irrelevance, has so far been defined in the literature only for the case of two variables. 
We define and discuss the \emph{independent natural extension} of a number of marginals. 
This is the most conservative joint model that arises out of the marginals and epistemic independence alone. 
Moreover, we show that the independent natural extension has a very important \emph{strong factorisation} property, which has a crucial part in our algorithm for updating credal trees under epistemic irrelevance.
\par
In Section~\ref{sec:joint}, we turn to the problem of constructing the most conservative global model based only on the local models in the tree and our Markov condition. 
We show that this task can be achieved by a recursive construction that proceeds from the leaves to the root of the tree using two operations: the \emph{independent natural extension} discussed in Section~\ref{sec:independent-natural-extension}, and the \emph{marginal extension}, defined and studied in \cite{walley1991,miranda2006b}. 
We also show that all uncertainty models we consider, the local ones as well as the global ones that we create, satisfy a consistency criterion that generalises (and is based on the same ideas as) the usual consistency criterion in Bayesian nets: they are (separately and jointly) \emph{coherent} \cite{walley1991,miranda2008a,miranda2009a,williams2007} (see in particular \cite[Section~8.1]{miranda2009b}). 
This is an important rationality requirement. 
\par
We briefly comment on some of the graphical separation criteria induced by epistemic irrelevance in Section~\ref{sec:separation}. 
We then go on to develop and justify an algorithm for making inferences on credal trees under epistemic irrelevance  in Section~\ref{sec:algorithm}. 
The algorithm is used to \emph{update} the tree: it computes posterior beliefs about a \emph{target} variable in the tree conditional on the observation of other variables, which are called \emph{instantiated}, meaning that their value is determined. 
It can in particular be used for treating the model as an expert system. 
\par
Our algorithm is based on message passing, as are the traditional algorithms that have been developed for precise graphical  models. 
It has some remarkable properties:
\begin{inparaenum}[(i)]
\item it works in time linear in the number of nodes in the tree;
\item it natively computes posterior lower and upper \emph{previsions} (or expectations) rather than probabilities;
\item it is the first algorithm developed for credal nets that exclusively uses the formalism of coherent lower previsions; and
\item it is shown that, under very mild conditions, using the tree for updating beliefs cannot lead to inferences that are inconsistent with the local models we have started from, nor with one another.
\end{inparaenum}
\par
We give a step-by-step example of the way inferences can be done using our algorithm in Section~\ref{sec:example}. 
We also comment there on the intriguing relationship between the failure of certain classical separation properties in our framework, and dilation \cite{herron1997,seidenfeld1993}. 
\par
The last part of the paper focuses on numerical simulations. 
In Section~\ref{sec:comparison} we empirically measure the amount of imprecision introduced by using epistemic irrelevance rather than strong  independence in a credal tree, when propagating inferences backwards (towards the root) from instantiated nodes to the target node. 
Indeed, it can be shown \cite{cooman2007d} that there is \emph{no difference} between inferences that go forward from instantiated nodes to the target node under strong independence and epistemic irrelevance. 
In Section~\ref{sec:application} we  present an application of our algorithm to on-line character recognition. 
We learn the probabilities from data and compare the predictions of our approach with those of its precise probability counterpart. 
The results are encouraging: they show that the tree can be used for real applications, and that the imprecision it originates is justified.
\par
In order to keep this paper reasonably short, we have to assume the reader has a good working knowledge of the basics of Peter Walley's \cite{walley1991} theory of coherent lower previsions. 
This is needed in particular for the most important proofs, collected in the Appendix. 
For a fairly detailed discussion of the coherence notions and results needed in the context of this paper, we refer to recent work by  Enrique Miranda \cite{miranda2008a,miranda2009a}. 

\section{Credal trees under epistemic irrelevance}
\label{sec:markov-trees}

\subsection{Basic notions and notation.}
We consider a rooted and directed discrete tree with finite width and depth. 
We call $\nodes$ the set of its nodes $s$, and we denote the \emph{root}, or initial, node by $\init$. 
For any node $s$, we denote its \emph{mother node} by $\mother{s}$. Of course, $\init$ has no mother node, and we use the convention $\mother\init=\emptyset$.
Also, for each node $s$, we denote the set of its \emph{children} by $\children{s}$, and the set of its \emph{siblings} by $\siblings{s}$. 
Clearly, $\siblings{\init}=\emptyset$, and if $s\neq\init$ then $\siblings{s}=\children{\mother{s}}\setminus\{s\}$. 
If $\children{s}=\emptyset$, then we call $s$ a \emph{leaf}, or \emph{terminal node}. 
We denote by $\nonterminals\coloneqq\set{s\in\nodes}{\children{s}\neq\emptyset}$ the set of all non-terminal nodes.
\par
For nodes $s$ and $t$, we write $s\precedes t$ if \emph{$s$ precedes $t$}, i.e., if there is a directed segment in the tree from $s$ to $t$. 
The relation $\precedes$ is a special partial order on the set $\nodes$.
$\ancestors{s}\coloneqq\set{t\in\nodes}{t\sprecedes s}$ denotes the chain of \emph{ancestors} of $s$, and
$\descendants{s}\coloneqq\set{t\in\nodes}{s\sprecedes t}$ its set of \emph{descendants}. 
Here $s\sprecedes t$ means that $s\precedes t$ and $s\neq t$.
We also use the notation $\until{s}\coloneqq\ancestors{s}\cup\{s\}$ for the chain (segment) connecting $\init$ and $s$, and $\after{s}\coloneqq\descendants{s}\cup\{s\}$ for the sub-tree with root $s$. 
Similarly, we let $\until{S}\coloneqq\bigcup\set{\until{s}}{s\in S}$ and $\after{S}\coloneqq\bigcup\set{\after{s}}{s\in S}$ for any subset $S\subseteq \nodes$. 
For any node $s$, its set of non-parent non-descendants is given by $\nonparnondes{s}\coloneqq\nodes\setminus(\msing{s}\cup\after{s})$.

\par
With each node $s$ of the tree, there is associated a variable $\var{s}$ assuming values in a non-empty finite set $\values{s}$. 
We denote  by $\gambles{s}$ the set of all real-valued maps (also called \emph{gambles}) on $\values{s}$. 
We extend this notation to more complicated situations as follows. 
If $S$ is any subset of $\nodes$, then we denote by $\var{S}$ the tuple of variables whose components are the $\var{s}$ for all $s\in S$. 
This new joint variable assumes values in the finite set
$\values{S}\coloneqq\times_{s\in S}\values{s}$, and the corresponding set of gambles is denoted by $\gambles{S}$. \footnote{\label{fn:empty-product} For any subset $S$ of $\nodes$, $\values{S}$ is defined formally as the set of all maps $\xval{S}$ of $S$ to $\bigcup_{s\in S}\values{s}$, such that $\xval{S}(s)=\xval{s}\in\values{s}$ for all $s\in S$. So when $S=\emptyset$, the empty product $\values{\emptyset}$ is defined as the set of all maps from $\emptyset$ to $\emptyset$, which is a singleton. The corresponding variable $\var{\emptyset}$ can then only assume this single value, so there is no uncertainty about it. $\gambles{\emptyset}$ can be identified with the set $\reals$ of real numbers.} 
Generic elements of $\values{s}$ are denoted by $\xval{s}$ or $\zval{s}$. 
Similarly for $\xval{S}$ and $\zval{S}$ in $\values{S}$. 
Also, if we mention a tuple $\zval{S}$, then for any $t\in S$, the corresponding element in the tuple will be denoted by $\zval{t}$. 
We assume all variables in the tree to be logically independent, meaning that the variable $\var{S}$ may assume \emph{all} values in $\values{S}$, for all $\emptyset\subseteq S\subseteq\nodes$. 
\par
We will frequently use the simplifying device of identifying a gamble $f_S$ on $\values{S}$ with its \emph{cylindrical extension} to $\values{U}$, where $S\subseteq U\subseteq\nodes$. 
This is the gamble $f_U$ on $\values{U}$ defined by $f_U(\xval{U})\coloneqq f_S(\xval{S})$ for all $\xval{U}\in\values{U}$. 
To give an example, if $\mathcal{K}\subseteq\gambles{\nodes}$, this trick allows us to consider $\mathcal{K}\cap\gambles{S}$ as the set of those gambles in $\mathcal{K}$ that depend only on the variable $\var{S}$. 
As another example, this device allows us to identify the gambles $\ind{\xsing{S}}$ and $\ind{\xsing{S}\times\values{\nodes\setminus S}}$, and therefore also the events $\xsing{S}$ and $\xsing{S}\times\values{\nodes\setminus S}$. 
More generally, for any event $A\subseteq\values{S}$, we can identify the gambles $\ind{A}$ and $\ind{A\times\values{\nodes\setminus S}}$, and therefore also the events $A$ and $A\times\values{\nodes\setminus S}$.
In the same spirit, a lower prevision on all gambles in $\gambles{S}$ can be identified with a lower prevision defined on the set of corresponding gambles on $\values{\nodes}$, a subset of $\gambles{\nodes}$.
\par
Throughout the paper, we consider (conditional) lower previsions as models for a subject's beliefs about the values that certain variables in the tree may assume.
We use a systematic notation for such (conditional) lower previsions. 
Let $I,O\subseteq\nodes$ be \emph{disjoint} sets of nodes with $O\neq\emptyset$, then we generically\footnote{Besides the letter $V$, we will also use the letters $P$, $Q$ and $R$.} denote by $\dlcon{O}{I}$ a \emph{conditional lower prevision}, defined on the set of gambles $\gambles{I\cup O}$.\footnote{In keeping with the observation in footnote~\ref{fn:empty-product}, we also allow $I=\emptyset$, which means conditioning on the variable $\var{I}=\var{\emptyset}$, which can only assume one single value. This means that $\dlcon{O}{\emptyset}\eqqcolon\dlpr_{O}$ effectively becomes an unconditional lower prevision on $\gambles{O\cup\emptyset}=\gambles{O}$. This a very useful device that allows us to use the same generic notation for both conditional and unconditional lower previsions.}
For every gamble $f$ on $\values{I\cup O}$ and every $\xval{I}\in\values{I}$, $\xindlcon[f]{O}{I}$ is the lower prevision (or lower expectation, or a subject's supremum buying price) for/of the gamble $f$, conditional on the event that $\var{I}=\xval{I}$.
We interpret $\dlcon[f]{O}{I}$ as a real-valued map (gamble) on $\values{I}$ that assumes the value $\xindlcon[f]{O}{I}$ in the element $\xval{I}$ of $\values{I}$.
The conjugate \emph{conditional upper prevision} $\ducon{O}{I}$ is defined on $\gambles{I\cup O}$ by $\ducon[f]{O}{I}\coloneqq-\dlcon[-f]{O}{I}$ for all gambles $f$ on $\values{I\cup O}$.
\par
We will always implicitly assume that all conditional models $\dlcon{O}{I}$ we use are \emph{separately coherent}, meaning that:
\begin{enumerate}[{\upshape SC}1.]
\item\label{item:separate:coherence:apg} $\xindlcon[f]{O}{I}\geq\min_{\zval{O}\in\values{O}}f(\xval{I},\zval{O})$ for all $f\in\gambles{I\cup O}$ and all $\xval{I}\in\values{I}$ [accepting partial gains];
\item\label{item:separate:coherence:superadd} $\xindlcon[f_1+f_2]{O}{I}\geq \xindlcon[f_1]{O}{I}+\xindlcon[f_2]{O}{I}$ for all $f_1,f_2\in\gambles{I\cup O}$ and all $\xval{I}\in\values{I}$ [super-additivity];
\item\label{item:separate:coherence:nnhom} $\xindlcon[\lambda f]{O}{I}=\lambda\xindlcon[f]{O}{I}$ for all $f\in\gambles{I\cup O}$, all non-negative real $\lambda$ and all $\xval{I}\in\values{I}$ [non-negative homogeneity].
\end{enumerate}
By combining SC\ref{item:separate:coherence:apg}--SC\ref{item:separate:coherence:nnhom}, it follows that
for all $f\in\gambles{I\cup O}$, $\xval{I}\in\values{I}$ and $\zval{O}\in\values{O}$:
\begin{equation*}
  \min_{\zval{O}\in\values{O}}f(\xval{I},\zval{O})
  \leq\xindlcon[f]{O}{I}
  \leq\xinducon[f]{O}{I}
  \leq\max_{\zval{O}\in\values{O}}f(\xval{I},\zval{O}).
\end{equation*}
If we let $f$ be the indicator $\ind{\zsing{I}}$ of the set $\zsing{I}$ in these inequalities, they reduce to the following, intuitively obvious, property:\footnote{For any event $A\subseteq\values{I\cup O}$, we denote $\xindlcon[\ind{A}]{O}{I}$ also as $\xindlcon[A]{O}{I}$ and call this real number the (conditional) lower \emph{probability} of $A$. Similarly $\xinducon[A]{O}{I}\coloneqq\xinducon[\ind{A}]{O}{I}$ is the (conditional) upper \emph{probability} of $A$.}
\begin{enumerate}[{\upshape SC}1.]
\addtocounter{enumi}{3}
\item\label{item:separate:coherence} $\xindlcon[{\zsing{I}\times\values{O}}]{O}{I}=\xinducon[{\zsing{I}\times\values{O}}]{O}{I}=\xindic{I}(\zval{I})$ for all $\xval{I},\zval{I}\in\values{I}$.
\end{enumerate}
From SC\ref{item:separate:coherence:apg}, SC\ref{item:separate:coherence:superadd} and SC\ref{item:separate:coherence} we can also derive that, with obvious notations:
\begin{equation}\label{eq:separate:coherence}
  \xindlcon[f]{O}{I}
  =\xindlcon[\xindic{I}f]{O}{I}
  =\xindlcon[f(\xval{I},\cdot)]{O}{I}
  \text{ for all gambles $f$ on $\values{I\cup O}$ and $\xval{I}\in\values{I}$},
\end{equation}
where $f(\xval{I},\cdot)$ is a partial map defined on $\values{O}$.
This implies that $\dlcon{O}{I}$ is completely determined by its behaviour on (cylindrical extensions of maps in) $\gambles{O}$.
\par
Hereafter, we will frequently introduce conditional lower previsions of the type $\dlcon{O}{I}$ as if they are defined on $\gambles{O}$, simply because that is a very natural thing to do: such a conditional lower prevision is usually interpreted as representing beliefs about the variable $\var{O}$, conditional on values of the variable $\var{I}$.
But the reader should keep in mind that, by the separate coherence property~\eqref{eq:separate:coherence}, $\dlcon{O}{I}$ can (and should) always be uniquely extended to the larger domain $\gambles{I\cup O}$.
\par
As soon as we consider a number of such conditional lower previsions $\dlcon{O_k}{I_k}$, $k=1,\dots,n$, they should satisfy more stringent consistency criteria than that each of them should be separately coherent: they should also be consistent with one another in the sense of Walley's \emph{(joint) coherence} \cite[Section~7.1.4(b)]{walley1991}. For more details about this much more involved type of coherence, we refer also to \cite{miranda2008a,miranda2009a}. 
\par
Finally, let us introduce one of the most important concepts for this paper, that of epistemic irrelevance. 
We describe the case of conditional irrelevance, as the unconditional version of epistemic irrelevance can easily be recovered as a special case.\footnote{It suffices, in the discussion below, to let $C=\emptyset$. As we indicated in footnote~\ref{fn:empty-product}, this makes sure the variable $\var{C}$ has only one possible value, so conditioning on that variable amounts to not conditioning at all.}
\par
Consider three disjoint subsets $C$, $I$, and $O$ of $N$, where both $I$ and $O$ are non-empty.
When a subject judges $\var{I}$ to be \emph{epistemically irrelevant to $\var{O}$ conditional on $\var{C}$}, he assumes that if he knows the value of $\var{C}$, then learning in addition which value $\var{I}$ assumes in $\values{I}$ will not affect his beliefs about $\var{O}$. 
More formally, assume that a subject has a separately coherent conditional lower prevision $\dlcon{O}{C}$ on $\gambles{O}$. 
If he assesses $\var{I}$ to be epistemically irrelevant to $\var{O}$ conditional on $\var{C}$, this implies that he can infer from his model $\dlcon{O}{C}$ a conditional model $\dlcon{O}{C\cup I}$ on $\gambles{O}$ given by 
\begin{equation*}
  \xindlcon[f]{O}{C\cup I}\coloneqq\xindlcon[f]{O}{C}
  \text{ for all $f\in\gambles{O}$ and all $\xval{C\cup I}\in\values{C\cup I}$.}
\end{equation*}

\subsection{Local uncertainty models.}
We now add a \emph{local uncertainty model} to each of the nodes $s$. 
If $s$ is not the root node, i.e.~has a mother $\mother{s}$, then this local model is a (separately coherent) conditional lower prevision $\llocconm{s}$ on $\gambles{s}$: for each possible value $\zval{\mother{s}}$ of the variable $\var{\mother{s}}$ associated with its mother $\mother{s}$, we have a coherent lower prevision $\zinllocconm{s}$ for the value of $\var{s}$, conditional on $\var{\mother{s}}=\zval{\mother{s}}$.
In the root, we have an unconditional local uncertainty model $\lloc{\init}$ for the value of $\var{\init}$. $\lloc{\init}$ is a (separately) coherent lower prevision on $\gambles{\init}$. 
We use the common generic notation $\llocconm{s}$ for all these local models.\footnote{We can do this because $\var{\mother{\init}}=\var{\emptyset}$ has only one possible value, so conditioning on that variable amounts to not conditioning at all.}

\subsection{Global uncertainty models.}\label{sec:criteria}
We intend to show in Section~\ref{sec:joint} how all these local models $\llocconm{s}$ can be combined into \emph{global uncertainty models}.  
We generically denote such global models using the letter $P$. 
More specifically, we want to end up with an unconditional joint lower prevision $\lpr\coloneqq\lpr_{\after{\init}}=\lpr_{\nodes}$ on $\gambles{\nodes}$ for all variables in the tree, as well as conditional lower previsions $\lglobcon{S}{s}$ on $\gambles{\after{S}}$ for all non-terminal nodes $s$ and all non-empty $S\subseteq\children{s}$.  
\par
\emph{Ideally, we want these global (conditional) lower previsions 
\begin{inparaenum}[(i)]
\item to be compatible with the local assessments  $\llocconm{s}$, $s\in\nodes$,
\item to be coherent with one another, and
\item to reflect the conditional irrelevancies (or Markov-type conditions) that we want the graphical structure of the tree to encode. 
  In addition, we want them 
\item to be as conservative (small) as possible.
\end{inparaenum}}
\par
In this list, the only item that needs more explanation concerns the Markov-type conditions that the tree structure encodes. 
This is what we turn to now.

\subsection{The interpretation of the graphical model.}\label{sec:graphical:interpretation}
In classical Bayesian nets, the graphical structure is taken to represent the following assessments: for any node $s$, conditional on its parent variables, its non-parent non-descendant variables are epistemically irrelevant to it (and therefore also independent).
\par
In the present context, we assume that the tree structure embodies the following conditional irrelevance assessment, which turns out to be equivalent with the conditional independence assessment above in the special case of a Bayesian tree.
\begin{enumerate}[{\upshape CI}.]
\item Consider any node $s$ in the tree, any subset $S$ of its set of children $\children{s}$, and the set $\nonparnondes{S}\coloneqq\bigcap_{c\in S}\nonparnondes{c}$ of their common non-parent non-descendants. Then \emph{conditional on the mother variable $\var{s}$, the non-parent non-descendant variables $\var{\nonparnondes{S}}$ are assumed to be epistemically irrelevant to the variables $\var{\after{S}}$ associated with the children in $S$ and their descendants.} 
\end{enumerate}
This interpretation turns the tree into a \emph{credal tree under epistemic irrelevance}, and we also introduce the term \emph{imprecise Markov tree} (IMT) for it. 
For the global models we are considering here, CI has the following consequences. 
It implies that for all $s\in\nonterminals$, all non-empty $S\subseteq\children{s}$ and all $I\subseteq\nonparnondes{S}$, we can infer from $\lglobcon{S}{s}$ a model $\lglobconirr{S}{s}{I}$, where for all $\zval{\sing{s}\cup I}\in\values{\sing{s}\cup I}$, with obvious notations:\footnote{For leaves $s$, the corresponding irrelevance condition is trivial, as the set $\children{s}$ of children of $s$ is empty.}
\begin{equation}\label{eq:irrelevance}
  \zinlglobconirr[f]{S}{s}{I}
  \coloneqq\zinlglobcon[f(\cdot,\zval{I})]{S}{s}
  \text{ for all gambles $f$ in $\gambles{\after{S}\cup I}$},
\end{equation}
where $f(\cdot,\zval{I})$ denotes a partial map of $f$, defined on $\values{\after{S}}$.
\par
We discuss some of the separation properties that accompany this interpretation in Section~\ref{sec:separation}. 
For now, we focus on two immediate consequences that will help us go from local to global models in Section~\ref{sec:joint}.
\par
First, consider some node $s$. 
Then CI tells us that for any two children $c_1,c_2\in\children{s}$ of $s$, the variable $\var{\after{c_1}}$ is epistemically irrelevant to the variable $\var{\after{c_2}}$, conditional on $\var{s}$. 
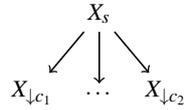
\begin{center}
  \begin{tikzpicture}\small
    \tikzstyle{edge from parent}+=[->,semithick]
    \node {$\var{s}$} 
    [grow=down,level distance=30,sibling distance=25]
    child {node {$\var{\after{c_1}}$}}
    child {node {$\dots$}}
    child {node {$\var{\after{c_2}}$}};
  \end{tikzpicture}
\end{center}
It even tells us that for any two disjoint non-empty sets $S_1\subseteq\children{s}$ and $S_2\subseteq\children{s}$ of children of $s$, the variable $\var{\after{S_1}}$ is epistemically irrelevant to $\var{\after{S_2}}$, conditional on $\var{s}$.
We conclude that, conditional on a node, all its children $c$ (and the variables associated with their sub-trees $\after{c}$) are \emph{epistemically independent} \cite[Chapter~9]{walley1991}, in the specific sense to be discussed in the next section.
\par
Next, consider some non-terminal node $s$ different from $\init$, and its mother variable $\var{\mother{s}}$. 
We infer from CI that this mother variable $\var{\mother{s}}$ is epistemically irrelevant to the variable $\var{\after{\children{s}}}$ conditional on $\var{s}$:
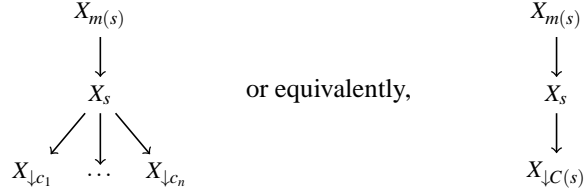
\begin{center}
  \begin{tikzpicture}\small
    \tikzstyle{edge from parent}+=[->,semithick]
    \node at (0,0) {$\var{\mother{s}}$} 
    [grow=down,level distance=30,sibling distance=25]
    child {node {$\var{s}$}
      child {node {$\var{\after{c_1}}$}}
      child {node {$\dots$}}
      child {node {$\var{\after{c_n}}$}}};
    \node at (3,-1) {\normalsize or equivalently, };
    \node at (6,0) {$\var{\mother{s}}$} 
    [grow=down,level distance=30]
    child {node {$\var{s}$}
      child {node {$\var{\after{\children{s}}}$}}};
  \end{tikzpicture}
\end{center}

\section{Independent natural extension}
\label{sec:independent-natural-extension}
Let us make a small digression on epistemic independence, which will help us in our discussion further on.
The material in this section is based on work that some of us have published elsewhere \cite{cooman2009b}, and we refer to that paper for more details and proofs for the results mentioned in this section. 

\subsection{Independent products.}\label{sec:indnatex}
Suppose we have a number of (separately) coherent marginal lower previsions $\mlpr{n}$ on $\gambles{n}$ representing beliefs about the values that each of a finite number of (logically independent) variables $\var{n}$ assume in the respective non-empty finite sets $\values{n}$, $n\in N$, where $N$ is some non-empty finite set.
\par
We want to construct a joint lower prevision $\mlpr{N}$ on $\gambles{N}$, where $\values{N}=\times_{n\in N}\values{n}$, that coincides with the marginals $\mlpr{n}$ on their respective domains $\gambles{n}$, and such that this $\mlpr{N}$ reflects the following structural assessments: for any disjoint proper subsets $O$ and $I$ of $N$, the variables $\var{I}$ are epistemically irrelevant to the variables $\var{O}$. 
In other words, learning the value of any number of these variables would not affect beliefs about the remaining variables. 
We then call the variables $\var{n}$, $n\in N$, \emph{epistemically independent}.
\par
Generally speaking, such irrelevance assessments are useful because they allow us to turn unconditional into conditional lower previsions. 
In particular, for any disjoint proper subsets $O$ and $I$ of $N$, we can use the epistemic irrelevance assessment of $\var{I}$ to $\var{O}$ to infer from the joint lower prevision $\mlpr{N}$ a conditional lower prevision $\clpr{O}{I}$ on $\gamblesio{O}{I}$ given by:
\begin{equation*}
  \zinclpr[h]{O}{I}
  \coloneqq\mlpr{N}(h(\cdot,\zval{I}))
  \text{ for all gambles $h$ on $\valuesio{O}{I}$ and all $\zval{I}\in\values{I}$}.
\end{equation*}
So we can use the symmetrised assessment of epistemic independence of the variables $\var{n}$, $n\in N$ to infer from $\mlpr{N}$ the following family of conditional lower previsions: 
\begin{equation*}
  \family{N}
  \coloneqq\set{\clpr{O}{I}}
  {\text{$O$ and $I$ disjoint proper subsets of $N$}}.
\end{equation*}
This idea leads to the definition of an independent product, which generalises the existing notion for (precise) probability models. 

\begin{definition}\label{def:independent:product}
  A (separately) coherent lower prevision $\mlpr{N}$ on $\gambles{N}$ that coincides with the marginal lower previsions $\mlpr{n}$ on their domains $\gambles{n}$, $n\in N$ and that is coherent with the family of conditional lower previsions $\family{N}$ is called an \emph{independent product}\footnote{In \cite{cooman2009b}, we distinguish between many-to-many and many-to-one independent products. It is not necessary to make this distinction here, but whenever we use the term `independent product' in the present paper, we implicitly refer to the more stringent many-to-many version introduced there.} of these marginals $\mlpr{n}$. 
\end{definition}

\noindent It turns out that there always is a point-wise smallest independent product:

\begin{proposition}\label{prop:independent-factorising}
  Any collection of (separately) coherent lower previsions $\mlpr{n}$ on $\gambles{n}$, $n\in N$, has a point-wise smallest  independent product. 
  We call it their \emph{independent natural extension} and denote it by $\otimes_{n\in N}\mlpr{n}$.
  Moreover, $\otimes_{n\in N}\mlpr{n}$ is a strongly factorising coherent lower prevision on $\gambles{N}$.
\end{proposition}

\noindent Strong factorisation is strongly linked with independent products, and will play a crucial part in our development of an algorithm for updating an imprecise Markov tree in Section~\ref{sec:algorithm}. 
It is defined as follows:

\begin{definition}
  We call a (separately) coherent lower prevision $\mlpr{N}$ on $\gambles{N}$
\emph{strongly factorising} if for all disjoint proper subsets $O$ and $I$ of $N$, all $g\in\gambles{O}$ and all non-negative $f\in\gambles{I}$, $\mlpr{N}(fg)=\mlpr{N}(f\mlpr{N}(g))$.
\end{definition}
\noindent
As another important example, the so-called \emph{strong product} $\times_{n\in N}\mlpr{n}$ \cite{cozman2000}  of marginal lower previsions $\mlpr{n}$ is strongly factorising.\footnote{This type of independent product comes to the fore in a study of credal nets under strong independence.} 
\par
As a consequence of the separate coherence of the joint lower prevision $\mlpr{N}$, the right-hand side of the equality in this definition can be rewritten as:
\begin{equation*}
  \mlpr{N}(f\mlpr{N}(g))
  =
  \begin{cases}
    \mlpr{N}(f)\mlpr{N}(g)
    &\text{ if $\mlpr{N}(g)\geq0$}\\
    \mupr{N}(f)\mlpr{N}(g)
    &\text{ if $\mlpr{N}(g)\leq0$},
  \end{cases}
\end{equation*}
which explains where the term `factorising' comes from.
In particular, for any (separately) coherent strongly factorising joint lower prevision $\mlpr{N}$, we see that 
for any partition $N_1$, \dots, $N_m$ of $N$:
\begin{equation}\label{eq:events-factorising}
  \mlpr{N}(\times_{k=1}^mA_k)=\prod_{k=1}^m\mlpr{N}(A_k)
  \text{ and }
  \mupr{N}(\times_{k=1}^mA_k)=\prod_{k=1}^m\mupr{N}(A_k),
\end{equation}
where $A_k\subseteq\values{I_k}$ for $k=1,\dots,m$. 
\par
The independent natural extension has very interesting and non-trivial \emph{marginalisation and associativity properties}. 
Consider any non-empty subset $R$ of $N$, then the independent natural extension $\otimes_{r\in R}\mlpr{r}$ of the marginals $\mlpr{r}$, $r\in R$ coincides with the restriction of $\otimes_{n\in N}\mlpr{n}$ to the set of gambles $\gambles{R}$:
\begin{equation}\label{eq:marginalisation}
  \left(\otimes_{r\in R}\mlpr{r}\right)(g)
  =\left(\otimes_{n\in N}\mlpr{n}\right)(g)
  \text{ for all gambles $g$ on $\values{R}$}.
\end{equation}
Moreover, for any partition $N_1$ and $N_2$ of $N$, we have that 
\begin{equation}\label{eq:associativity}
  \otimes_{n\in N}\mlpr{n}
  =\left(\otimes_{n_1\in N_1}\mlpr{n_1}\right)\otimes\left(\otimes_{n_2\in N_2}\mlpr{n_2}\right),
\end{equation}
so $\otimes_{n\in N}\mlpr{n}$ is the independent natural extension of its $\values{N_1}$-marginal $\otimes_{n_1\in N_1}\mlpr{n_1}$ and its  $\values{N_2}$-marginal $\otimes_{n_2\in N_2}\mlpr{n_2}$.

\subsection{Regular extension}
As a next step, suppose we want to condition a separately coherent and strongly factorising joint $\mlpr{N}$ on observations of the type $\var{I}=\zval{I}$, where $I$ is some proper subset of $N$.
In other words, we want to find conditional lower previsions $\clpr{O}{I}$ on $\gamblesio{I}{O}$ that are (jointly) coherent with the joint lower prevision $\mlpr{N}$.
To this end, we calculate the so-called \emph{regular extension} as follows. 
Consider $\zval{I}$ in $\values{I}$. 
When $\mupr{N}(\zsing{I})>0$,
\begin{equation*}
  \zinlrext[h]{I}
  \coloneqq\max\set{\mu\in\reals} {\mlpr{N}(\zindic{I}[h-\mu])\geq0},
\end{equation*}
where $O$ is any non-empty subset of $N\setminus I$ and $h$ is any gamble on $\valuesio{I}{O}$. 
When $\mupr{N}(\zsing{I})=0$, $\xinlrext{I}$ is \emph{vacuous}, meaning that $\zinlrext[h]{I}=\min_{\xval{O}\in\values{O}}h(\zval{I},\xval{O})$ for all gambles $h$ on $\valuesio{I}{O}$.
\par
Generally speaking, coherence only determines $\zinclpr{O}{I}$ uniquely if  $\mlpr{N}(\zsing{I})>0$, and in that case regular extension yields this uniquely coherent conditional lower prevision: $\zinclpr{O}{I}=\zinlrext{I}$.
When $\mlpr{N}(\zsing{I})=0$, regular extension is still coherent, and it even still characterises the coherent $\zinclpr{O}{I}$, because these all lie between the vacuous lower prevision and $\zinlrext{I}$.
For more details about this regular extension, we refer to \cite[Appendix~J]{walley1991} and \cite[Section~4]{miranda2009a}. 
\par
If the joint $\mlpr{N}$ is strongly factorising, we get:
\begin{align*}
  \mlpr{N}(\xindic{I}[h-\mu])
  &=\mlpr{N}(\xindic{I}\mlpr{N}(h(\xval{I},\cdot)-\mu))\\
  &=
  \begin{cases}
    \mlpr{N}(\xsing{I})[\mlpr{N}(h(\xval{I},\cdot))-\mu]
    &\text{ if $\mlpr{N}(h(\xval{I},\cdot))\geq\mu$}\\
    \mupr{N}(\xsing{I})[\mlpr{N}(h(\xval{I},\cdot))-\mu]
    &\text{ if $\mlpr{N}(h(\xval{I},\cdot))\leq\mu$},
  \end{cases}
\end{align*}
so we conclude that, quite interestingly,
\begin{equation}\label{eq:regular-extension-factorising}
  \xinlrext[h]{I}=\mlpr{N}(h(\xval{I},\cdot))
  \text{ as soon as $\mupr{N}(\xsing{I})>0$}.
\end{equation}
In other words, the conditional lower previsions found by regular extension of a strongly factorising joint satisfy all epistemic irrelevance conditions present in an assessment of epistemic independence.
We shall have occasion to use this idea several times in the course of this paper, especially in the proofs.

\subsection{Conditionally independent products.}\label{sec:condindnatex}
To end this section, we generalise the notion of an independent product to that of a conditionally independent product. 
In this case we have a number of `marginal' conditional lower previsions $\mlprtoo{n}$ on $\gambles{n}$ representing beliefs (conditional on a variable $\vartoo$ in a finite set $\toovalues$) about the values that each of a finite number of (logically independent) variables $\var{n}$ assume in the respective non-empty finite sets $\values{n}$, $n\in N$.
\par
We want to construct a conditional lower prevision $\mlprtoo{N}$ on $\gambles{N}$, where $\values{N}=\times_{n\in N}\values{n}$, that coincides with the marginal conditional lower previsions $\mlprtoo{n}$ on their respective domains $\gambles{n}$, and such that this $\mlprtoo{N}$ reflects the following structural assessments: for any disjoint proper subsets $O$ and $I$ of $N$, the variables $\var{I}$ are epistemically irrelevant to the variables $\var{O}$, \emph{conditional on $\vartoo$}. 
In other words, if the value of $\vartoo$ was known, then learning the value of any number of these variables would not affect beliefs about the remaining variables. 
We then call the variables $\var{n}$, $n\in N$ \emph{epistemically independent, conditional on $\vartoo$}.
\par
Generally speaking, such conditional irrelevance assessments are useful because they allow us to turn lower previsions conditional on $\vartoo$ alone into other, more involved conditional lower previsions. 
In particular, for any disjoint proper subsets $O$ and $I$ of $N$, we can use the epistemic irrelevance assessment of $\var{I}$ to $\var{O}$ conditional on $\vartoo$ to infer from the joint lower prevision $\mlprtoo{N}$ a conditional lower prevision $\clprtoo{O}{I}$ on $\gamblesio{O}{I}$ [or equivalently on $\gamblesiotoo{O}{I}$] given by:
\begin{equation*}
  \zinclprtoo[h]{O}{I}
  \coloneqq\inmlprtoo[h(\cdot,\zval{I})]{N}
  \text{ for all gambles $h$ on $\valuesio{O}{I}$ and all $\zval{I}\in\values{I}$}.
\end{equation*}
So we can use the symmetrised assessment of epistemic independence of the variables $\var{n}$, $n\in N$ conditional on $\vartoo$ to infer from the $\mlprtoo{N}$ the following family of conditional lower previsions: 
\begin{equation*}
  \familytoo{N}
  \coloneqq\set{\clprtoo{O}{I}}
  {\text{$O$ and $I$ disjoint proper subsets of $N$}}.
\end{equation*}
This idea leads to the definition of a conditionally independent product. 

\begin{definition}\label{def:conditionally:independent:product}
  A (separately) coherent conditional lower prevision $\mlprtoo{N}$ on $\gambles{N}$ that coincides with the `marginal' conditional lower previsions $\mlprtoo{n}$ on their domains $\gambles{n}$, $n\in N$ and that is coherent with the family of conditional lower previsions $\familytoo{N}$ is called a \emph{conditionally independent product} of these marginals $\mlprtoo{n}$.
\end{definition}

\noindent It turns out that there always is a point-wise smallest conditionally independent product:

\begin{proposition}\label{prop:conditionally:independent-factorising}
  Any collection of (separately) coherent conditional lower previsions $\mlprtoo{n}$ on $\gambles{n}$, $n\in N$, has a point-wise smallest conditionally independent product. 
  We call it their \emph{conditionally independent natural extension} and denote it by $\otimes_{n\in N}\mlprtoo{n}$.
\end{proposition}

\noindent The notation we use for the conditionally independent natural extension is appropriately suggestive: for each $\valtoo$ in $\toovalues$, $\otimes_{n\in N}\inmlprtoo{n}$ is indeed the independent natural extension of the marginal lower previsions $\inmlprtoo{n}$.
This implies that each $\otimes_{n\in N}\inmlprtoo{n}$ is a strongly factorising coherent lower prevision on $\gambles{N}$.
\par
We are now ready to go back to our discussion of imprecise Markov trees.

\section{Constructing the most conservative joint}
\label{sec:joint}
Let us show how to construct specific global models for the variables in the tree, and argue that these are the most conservative coherent models that extend the local models and express all conditional irrelevancies~\eqref{eq:irrelevance}, encoded in the imprecise Markov tree. 
In Section~\ref{sec:algorithm}, we will use these global models to construct and justify an algorithm for updating the imprecise Markov tree. 
\par
The crucial step lies in the recognition that any tree can be constructed recursively from the leaves up to the root, by using basic building blocks of the following type:
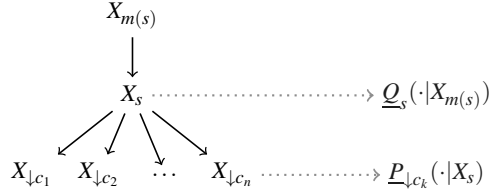
\begin{center}
  \begin{tikzpicture}\small
    \tikzstyle{edge from parent}+=[->,semithick]
    \tikzstyle{infoline}=[->,thick,dotted,draw=blue!50]
    \node (xm) 
    {$\var{\mother{s}}$}
    [grow=down,level distance=30,sibling distance=25]
    child {node (xs) {$\var{s}$}
      child {node (xc1) {$\var{\after{c_1}}$}}
      child {node (xc2) {$\var{\after{c_2}}$}}
      child {node (dots) {$\dots$}}
      child {node (xcn) {$\var{\after{c_n}}$}}};
    \node[right of=xs,node distance=4cm,color=blue] (xsmodel) 
    {$\llocconm{s}$};
    \node[below of=xsmodel,node distance=30,color=blue] (xckmodel) 
    {$\lglobcon{c_k}{s}$};
    \path[infoline] (xs) -- (xsmodel);
    \path[infoline] (xcn) -- (xckmodel);
  \end{tikzpicture}
\end{center}
The global models are then also constructed recursively, following the same pattern. 
In what follows, we first derive the recursion equations for these global models in a heuristic manner.
The real justification for using the global models thus derived is then given in Theorem~\ref{theo:global-models}.
\par
Consider a node $s$ and suppose that, in each of its children  $c\in\children{s}$, we already have a global conditional lower prevision $\lglobcon{c}{s}$ on $\gambles{\after{c}}$ [or equivalently, on $\gambles{\sing{s}\cup\after{c}}$]. 
\par
Given that, conditional on $\var{s}$,  the variables $\var{\after{c}}$, $c\in\children{s}$ are epistemically independent [see Section~\ref{sec:graphical:interpretation}, condition~CI], the discussion in Section~\ref{sec:independent-natural-extension} leads us to combine the `marginals' $\lglobcon{c}{s}$, $c\in\children{s}$ into their point-wise smallest conditionally independent product (conditionally independent natural extension)  $\otimes_{c\in\children{s}}\lglobcon{c}{s}$, which is a conditional lower prevision $\linecon{s}$ on $\gambles{\after{\children{s}}}$ [or equivalently, on $\gambles{\after{s}}$]: 
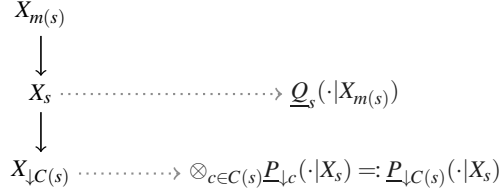
\begin{center}
  \begin{tikzpicture}\small
    \tikzstyle{edge from parent}+=[->,semithick]
    \tikzstyle{infoline}=[->,thick,dotted,draw=blue!50]
    \node (xm) 
    {$\var{\mother{s}}$}
    [grow=down,level distance=30,sibling distance=25]
    child {node (xs) {$\var{s}$}
      child {node (xc) {$\var{\after{\children{s}}}$}}};
    \node[right of=xs,node distance=4cm,color=blue] (xsmodel) 
    {$\llocconm{s}$};
    \node[below of=xsmodel,node distance=30,color=blue] (xcmodel) 
    {$\otimes_{c\in\children{s}}\lglobcon{c}{s}\eqqcolon\linecon{s}$};
    \path[infoline] (xs) -- (xsmodel);
    \path[infoline] (xc) -- (xcmodel);
  \end{tikzpicture}
\end{center}
\par
Next, we need to combine the conditional models $\llocconm{s}$ and  $\linecon{s}$ into a global conditional model about $\var{\after{s}}$.
Given that, conditional on $\var{s}$, the variable $\var{\mother{s}}$ is epistemically irrelevant to the variable $\var{\after{\children{s}}}$ [see Section~\ref{sec:graphical:interpretation}, condition~CI], we expect $\lpr_{\after{\children{s}}}(\cdot\vert\var{\sing{\mother{s},s}})$ and $\linecon{s}$ to coincide [this is a special instance of Eq.~\eqref{eq:irrelevance}]. 
The most conservative (point-wise smallest) coherent way of combining the conditional lower previsions 
$\lpr_{\after{\children{s}}}(\cdot\vert\var{\sing{\mother{s},s}})$ and $\llocconm{s}$ consists in taking their \emph{marginal extension}\footnote{Marginal extension is, in the special case of precise probability models, also known as the law of total probability, or the law or iterated expectations.} $\llocconm[{\lpr_{\after{\children{s}}}(\cdot\vert\var{\sing{\mother{s},s}})}]{s}=\llocconm[\linecon{s}]{s}$; see \cite{miranda2006b,walley1991} for more details. 
Graphically: 
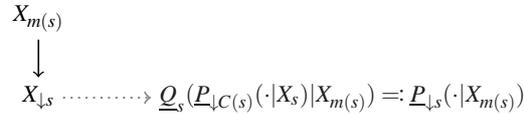
\begin{center}
  \begin{tikzpicture}\small
    \tikzstyle{edge from parent}+=[->,semithick]
    \tikzstyle{infoline}=[->,thick,dotted,draw=blue!50]
    \node (xm) 
    {$\var{\mother{s}}$}
    [grow=down,level distance=30,sibling distance=25]
    child {node (xs) {$\var{\after{s}}$}};
    \node[right of=xs,node distance=4cm,color=blue] (xsmodel) 
    {$\llocconm[\linecon{s}]{s}\eqqcolon\lglobconm{s}$};
    \path[infoline] (xs) -- (xsmodel);
  \end{tikzpicture}
\end{center}
\par
Summarising, and also accounting for the case $s=\init$, we can construct a global conditional lower prevision $\lglobconm{s}$ on $\gambles{\after{s}}$ by backwards recursion:
\begin{align}
  \linecon{s}
  &\coloneqq\otimes_{c\in\children{s}}\lglobcon{c}{s}
  \label{eq:backwards-nine}\\
  \lglobconm{s}
  &\coloneqq\llocconm[\linecon{s}]{s}  
  =\llocconm[\otimes_{c\in\children{s}}\lglobcon{c}{s}]{s},
  \label{eq:backwards-global}
\end{align}
for all $s\in\nonterminals$. 
If we start with the `boundary conditions'
\begin{equation}\label{eq:boundary-condition}
  \lglobconm{t}\coloneqq\llocconm{t}
  \text{ for all leaves $t$}, 
\end{equation}
then the recursion relations~\eqref{eq:backwards-nine} and~\eqref{eq:backwards-global} eventually lead to the global joint model $\lglob{\init}=\lglobconm{\init}$, and to the global conditional models $\linecon{s}$ for all non-terminal nodes~$s$.
For any subset $S\subseteq\children{s}$, the global conditional model $\lglobcon{S}{s}$ can then be defined simply as the restriction of the model $\linecon{s}$ on $\gambles{\after{\children{s}}}$ to the set $\gambles{\after{S}}$:
\begin{equation}\label{eq:backwards-restrict}
  \lglobcon[g]{S}{s}\coloneqq \linecon[g]{s}
  \text{ for all gambles $g$ on $\values{\after{S}}$}.
\end{equation}
It follows from the discussion in Section~\ref{sec:independent-natural-extension} that, alternatively [see Eq.~\eqref{eq:marginalisation}],
\begin{equation}\label{eq:backwards-nine-restrict}
  \lglobcon{S}{s}
  =\otimes_{c\in S}\lglobcon{c}{s}.
\end{equation}
For easy reference, we will in what follows refer to this collection of global models as the \emph{family of global models $\treefamily{P}$,} so
\begin{equation*}
  \treefamily{P}
  \coloneqq\{\lglob{}\}\cup
  \set{\lglobcon{S}{s}}{s\in\nonterminals\text{ and non-empty }S\subseteq\children{s}}.
\end{equation*}
\par
We end this section by discussing a number of interesting properties for the family of global models $\treefamily{P}$ we can derive in this way.
Let us call any real functional $\Phi$ on $\gambles{}$ \emph{strictly positive} if $\Phi(\xindic{})>0$ for all $\xval{}\in\values{}$. 

\begin{proposition}\label{prop:global:positivity}
  If all the local models $\ulocconm{s}$, $s\in\nodes$ are strictly positive, then so are all the global models in $\treefamily{P}$. 
\end{proposition}

\begin{proposition}\label{prop:global:positivity:too}
  Consider any non-empty subset $E$ of $\nodes$ and any $\xval{E}\in\values{E}$. 
If\/ $\upr(\xsing{E})>0$ then also $\xinuglobcon[\xsing{E\cap\after{c}}]{c}{e}>0$ for all $e\in E$ and all $c\in\children{e}$.\footnote{Observe that this holds trivially also if $E\cap\after{c}=\emptyset$, because then $\values{E\cap\after{c}}=\values{\emptyset}$ is a singleton [see footnote~\ref{fn:empty-product}] whose upper probability should be $1$ by separate coherence.}
\end{proposition}

Before we formulate the most important result in this section (and arguably, in this paper), we provide some motivation.
Suppose we have some family of global models 
\begin{equation*}
  \treefamily{V}
  \coloneqq\{\dlglob{}\}\cup
  \set{\dlglobcon{S}{s}}{s\in\nonterminals\text{ and non-empty }S\subseteq\children{s}}
\end{equation*}
associated with the tree.
How do we express that such a family is compatible with the assessments encoded in the tree? 
\par
First of all, we require that our global models should extend the local models:
\begin{enumerate}[{\upshape T}1.]
\item\label{item:global:local} For each $s\in\nodes$, $\llocconm{s}$ is the restriction of $\dlglobconm{s}$ to $\gambles{s}$.
\end{enumerate}
The second requirement is that our models should satisfy the rationality requirement of coherence:
\begin{enumerate}[{\upshape T}1.]
\addtocounter{enumi}{1} 
\item\label{item:global:coherence} The (conditional) lower previsions in $\treefamily{V}$ are jointly coherent.
\end{enumerate}
The third requirement requires more explanation: the global models should reflect all epistemic irrelevancies encoded in the graphical structure of the tree.  
Naively, we would want condition~\eqref{eq:irrelevance} to be satisfied. 
The problem is that only the right-hand side in Eq.~\eqref{eq:irrelevance}, involving the model $\dlglobcon{S}{s}$ is directly available to us.
To get to the left-hand side involving the model $\dlglobconirr{S}{s}{I}$, one naive approach would be to `condition the joint model $\dlglob{}=\dlglob{T}$ on the variable $\var{\sing{s}\cup I}$'. 
But we have seen in Section~\ref{sec:indnatex} that given a joint model, coherence in general only determines the conditional models uniquely, provided that the \emph{lower probability} of the conditioning event is non-zero. 
This is a fairly strong condition, and in what follows we would generally prefer to work with the much weaker condition that the \emph{upper probability} of the conditioning event is non-zero.\footnote{As the results in \cite{cooman2009b} suggest, it might be possible to go even further, and prove a counterpart to Theorem~\ref{theo:global-models} with no positivity restrictions on the local models. We leave this as an avenue for future research, however.} Since in that case the left-hand side of  Eq.~\eqref{eq:irrelevance} need not be uniquely determined from the joint $\dlglob{}$ by coherence, this approach becomes unfeasible.
\par
Nevertheless, as soon as we realise that all we can reasonably require from our models is that they should be coherent, the right approach readily suggests itself:\footnote{This is also the approach implicit in Definition~\ref{def:independent:product}, as well as the one used in \cite{cooman2009b}. It coincides with the usual, naive approach as soon as all the relevant conditional models are uniquely determined from the joint by coherence.} we should require that if we  use the available models $\dlglobcon{S}{s}$ to \emph{define} the models $\dlglobconirr{S}{s}{I}$ through the epistemic irrelevance condition~\eqref{eq:irrelevance}, then the result should still be coherent:
\begin{enumerate}[{\upshape T}1.]
\addtocounter{enumi}{2} 
\item\label{item:global:irrelevance} If we define the conditional lower previsions $\dlglobconirr{S}{s}{I}$, $s\in\nonterminals$, $S\subseteq\children{s}$ and $R\subseteq\nonparnondes{S}$ through the epistemic irrelevance requirements
\begin{equation*}
 \zindlglobconirr[f]{S}{s}{R}
 \coloneqq\zindlglobcon[f(\cdot,\zval{R})]{S}{s}
 \text{ for all gambles $f$ in $\gambles{\after{S}\cup R}$},
\end{equation*}
then all these models together should be (jointly) coherent with all the available models in the family $\treefamily{V}$.
\end{enumerate}
And there is a final requirement, which guarantees that all inferences we make on the basis of our global models are as conservative as possible, and are therefore based on no other considerations than what is encoded in the tree:
\begin{enumerate}[{\upshape T}1.]
\addtocounter{enumi}{3} 
\item\label{item:global:smallest} The models in the family $\treefamily{V}$ are dominated (point-wise) by the corresponding models in all other families satisfying requirements~{\upshape T\ref{item:global:local}}--{\upshape T\ref{item:global:irrelevance}}.
\end{enumerate}
\par
It turns out that the family of models $\treefamily{P}$ we have been constructing above satisfy all these requirements.

\begin{theorem}\label{theo:global-models}
  If all local models $\ulocconm{s}$ on $\gambles{s}$, $s\in\nodes$ are strictly positive, then the family of global models $\treefamily{P}$, obtained through Eqs.~\eqref{eq:backwards-nine}--\eqref{eq:backwards-restrict}, constitutes the point-wise smallest family of (conditional) lower previsions that satisfy {\upshape T\ref{item:global:local}}--{\upshape T\ref{item:global:irrelevance}}. 
  It is therefore the unique family to also satisfy {\upshape T\ref{item:global:smallest}}. 
  Finally, consider any non-empty set of nodes $E\subseteq\nodes$ and the corresponding conditional lower prevision derived by applying regular extension:\footnote{If we look at the proof of this result in the Appendix, it is not hard to see that similar statements can be made about the (joint) coherence of the regular extensions $\lrext{E_k}$ for any finite collection $E_k$, $k=1,\dots,n$ of sets of nodes.}
  \begin{equation*}
    \xinlrext[f]{E}
    \coloneqq\max\set{\mu\in\reals}{\lglob{\after{T}}(\xindic{E}[f-\mu])\ge0}
    \text{ for all $f\in\gambles{T}$ and all $\xval{E}\in\values{E}$}.
  \end{equation*}
  Then the conditional lower prevision $\lrext{E}$ is (jointly) coherent with the global models in the family $\treefamily{P}$.
\end{theorem}
\noindent
The last statement of this theorem guarantees that if we use regular extension to \emph{update the tree} given evidence $\var{E}=\xval{E}$, i.e., derive conditional models $\xinlrext{E}$ from the joint model $\lglob{}=\lglob{\after{T}}$, such inferences will always be coherent.
This is of particular relevance for the discussion in Section~\ref{sec:algorithm}, where we derive an efficient algorithm for updating the tree using regular extension.
It implies in particular that our algorithm produces coherent inferences.

\section{Some separation properties}
\label{sec:separation}
Without going into too much detail, we would like to point out some of the more striking differences between the separation properties in imprecise Markov trees under epistemic irrelevance, and the more usual ones that are valid for Bayesian nets \cite{pearl1988}, which, by the way, are also inherited from Bayesian nets by credal nets under strong independence \cite{cozman2000}.
\par
It is clear from the interpretation of the graphical model described in Section~\ref{sec:graphical:interpretation} that we have the following simple separation results:
\begin{center}
  \begin{tikzpicture}\small
    \tikzstyle{edge from parent}+=[->,semithick]
    \node (xi1)  {$\var{i_1}$}
    [grow=right,sibling distance=25]
    child {node (xi2) {$\var{i_2}$}
      child {node (xt) {$\var{t}$}}};
  \end{tikzpicture}
  \qquad
  \begin{tikzpicture}\small
    \tikzstyle{edge from parent}+=[->,semithick]
    \node (xi2)  {$\var{i_2}$}
    child[grow=left,sibling distance=25] {node (xi1) {$\var{i_1}$}}
    child[grow=right,sibling distance=25] {node (xt) {$\var{t}$}};
  \end{tikzpicture}
\end{center}
where in both cases, \emph{$\var{i_2}$ separates $\var{t}$ from $\var{i_1}$}: when the value of $\var{i_2}$ is known, additional information about the value of $\var{i_1}$ does not affect beliefs about the value of $\var{t}$. 
In this figure, between $i_1$ and $i_2$, and between $i_2$ and $t$, there may be other nodes, but the arrows along the path segment through these nodes should all point in the indicated directions. 
The underlying idea is that $t$ is a (descendant of some) child $c$ of $i_2$, and conditional on the mother $i_2$ of $c$, the non-parent non-descendant $i_1$ of $c$ is epistemically irrelevant to $c$ and all of its descendants.
\par
On the other hand, and in contradistinction with what we are used to in Bayesian nets, we will not generally have separation in the following configuration:
\begin{center}
  \begin{tikzpicture}\small
    \tikzstyle{edge from parent}+=[<-,semithick]
    \node (xi1)  {$\var{i_1}$}
    [grow=right,sibling distance=25]
    child {node (xi2) {$\var{i_2}$}
      child {node (xt) {$\var{t}$}}};
  \end{tikzpicture}
\end{center}
where \emph{$\var{i_2}$ does not necessarily separate $\var{t}$ from $\var{i_1}$}. 
We will come across a simple counterexample in Section~\ref{sec:example}. 
Where does this difference with the case of Bayesian nets originate? 
It is clear from the reasoning above that $\var{i_2}$ separates $\var{i_1}$ from $\var{t}$: conditional on $\var{i_2}$, $\var{t}$ is epistemically irrelevant to $\var{i_1}$. 
For precise probability models,  irrelevance generally implies symmetrical independence, and therefore  this will generally imply that conditional on $\var{i_2}$, $\var{i_1}$ is epistemically irrelevant to $\var{t}$  as well. 
But for imprecise probability models no such symmetry is guaranteed \cite{couso1999b}, and we therefore cannot infer that, generally speaking, $\var{i_2}$ will separate $\var{i_1}$ from $\var{t}$. 
\emph{As a general rule, we can only infer separation if the arrows point from the `separating' variable $\var{i_2}$ towards the `target' variable $\var{t}$.} 

\section{A fast algorithm for updating in an imprecise Markov tree}
\label{sec:algorithm}
We now consider the case where we are interested in making inferences about the value of the variable $\var{t}$ in some \emph{target node} $t$, when we know the values $\xval{E}$ of the variables $\var{E}$ in a set $E\subseteq\nodes\setminus\{t\}$ of \emph{evidence nodes}; see for instance Fig~\ref{fig:mtree}.

\subsection{The formulation of the problem.}
If we assume that the values of the remaining variables are \emph{missing at random}, then we can do this by conditioning the joint $\lpr$ obtained above on the available evidence `$\var{E}=\xval{E}$'; see for instance \cite{cooman2004b,zaffalon2009}. 
\par
We will address this problem by updating the lower prevision $\lpr$ to the lower prevision $\lupcon{t}{E}$ on $\gambles{t}$ using \emph{regular extension}  \cite[Appendix~J]{walley1991}:
\begin{equation}\label{eq:regular}
  \lupcon[g]{t}{E}=\max\set{\mu\in\reals}{\lpr(\xindic{E}[g-\mu])\geq0}
\end{equation}
for all gambles $g$ on $\values{t}$, \emph{assuming that $\upr(\xsing{E})>0$}. 
Theorem~\ref{theo:global-models} guarantees that such inferences are coherent.
Sufficient conditions on the local models for this positivity assumption to hold are given in Proposition~\ref{prop:global:positivity}.
\par
Consider the map
\begin{equation*}
  \rho_g\colon\reals\to\reals\colon\mu\mapsto\lpr(\xindic{E}[g-\mu]).
\end{equation*}
We can infer from the separate coherence of $\lpr$ that $\abs{\rho_g(\mu_1)-\rho_g(\mu_2)}\leq\abs{\mu_1-\mu_2}\upr(\{\xval{E}\})$ for all $\mu_1,\mu_2\in\reals$, which implies that $\rho_g$ is (Lipschitz) continuous. 
Separate coherence of $\lpr$ also guarantees that $\rho_g$ is concave and non-increasing. 
Hence $\set{\mu\in\reals}{\rho_g(\mu)\geq0}=(-\infty,\lupcon[g]{t}{E}]$, which shows that the supremum that we should have {\itshape a priori} used in~\eqref{eq:regular} is indeed a maximum. 
$\lupcon[g]{t}{E}$ is the right-most zero of $\rho_g$, and it is, again by separate coherence of $\lpr$, guaranteed to lie between the smallest value $\min g$ and the largest value $\max g$ of $g$. 
If moreover $\lpr(\xsing{E})>0$, then separate coherence of $\lpr$ implies that $\lupcon[g]{t}{E}$ is the unique zero of $\rho_g$. 
If on the other hand $\lpr(\xsing{E})=0$, then $(-\infty,\lupcon[g]{t}{E}]$ is the set of all zeros of $\rho_g$. 
It appears that any algorithm for calculating $\lupcon[g]{t}{E}$ will benefit from being able to calculate the values of $\rho_g$, or even more simply check their signs, efficiently.

\subsection{Calculating the values of $\rho_g$ recursively.}\label{sec:calculating-rho}
We now recall from Section~\ref{sec:joint} that the joint $\lpr$ can be constructed recursively from leaves to root. 
The idea we now use is that calculating $\rho_g(\mu)=\lpr(\xindic{E}[g-\mu])$ becomes easier if we graft the structure of the tree onto the argument $g^\mu\coloneqq\xindic{E}[g-\mu]$ as follows. 
Define
\begin{equation*}
  g^\mu_s
  \coloneqq
  \begin{cases}
    \xindic{s}&\text{if $s\in E$}\\
    g-\mu&\text{if $s=t$}\\
    1&\text{if $s\in\nodes\setminus(E\cup\{t\})$},
  \end{cases}
\end{equation*}
then $g^\mu_s\in\gambles{s}$ and $g^\mu=\prod_{s\in\nodes}g^\mu_s$. 
Also define, for any $s\in\nodes$, the gamble $\phi^\mu_s$ on $\values{\after{s}}$ by $\phi^\mu_s\coloneqq\prod_{u\in\after{s}}g^\mu_u$. 
Then 
\begin{equation*}
  \phi^\mu_\init=g^\mu
  \text{ and }
  \phi^\mu_s\geq0\text{ if $s\nprecedes t$},
\end{equation*}
and
\begin{equation}\label{eq:function-branching}
  \phi^\mu_s=g^\mu_s\prod_{c\in\children{s}}\phi^\mu_c
  \text{ for all $s\in\nodes$},
\end{equation}
where we use the convention that any product over an empty set of indices equals one.
Eq.~\eqref{eq:function-branching} is the argument counterpart of Eq.~\eqref{eq:backwards-global}. 
Also, if $s\nprecedes t$ then $g^\mu_s$ and $\phi^\mu_s$ do  not depend on $\mu$, nor on $g$.
Indeed, in that case
\begin{equation}\label{eq:function-branching-too}
  \phi^\mu_s=\xindic{E\cap\after{s}}.
\end{equation}

\subsubsection*{First, let us consider the nodes $s\nprecedes t$.} 
We define the \emph{messages} $\lmess{s}$ and $\umess{s}$ recursively by
\begin{equation}\label{eq:lower-upper-messages}
  \lmess{s}
  \coloneqq\bigllocconm[g^\mu_s
  \smashoperator{\prod_{c\in\children{s}}}\lmess{c}]{s}\\
  \text{ and }
  \umess{s}
  \coloneqq\bigulocconm[g^\mu_s
  \smashoperator{\prod_{c\in\children{s}}}\umess{c}]{s}.
\end{equation}
We summarise such a pair by the notation: $\lumess{s}\coloneqq\lulocconm[g^\mu_s\prod_{c\in\children{s}}\lumess{c}]{s}\coloneqq(\lmess{s},\umess{s})$.
Then there are two possibilities:
\begin{equation*}
  \lumess{s}=
  \begin{cases}
    \lulocconm[\xsing{s}]{s}\smashoperator{\prod_{c\in\children{s}}}\lumess{c}(\xval{s})
    &\text{ if $s\in E$}\\
    \biglulocconm[{\smashoperator[r]{\prod_{c\in\children{s}}}\lumess{c}}]{s}
    &\text{ if $s\notin E$}.
  \end{cases}
\end{equation*}
The messages $\lmess{s}$ and $\umess{s}$ are gambles on $\values{\mother{s}}$, and can therefore be seen as tuples of real numbers, with as many components $\lumess{s}(\xval{\mother{s}})$ as there are elements $\xval{\mother{s}}$ in $\values{\mother{s}}$. 
They are all non-negative.
As their notation suggests, they do not depend on the choice of $g$ or $\mu$, but only (at most) on which nodes are \emph{instantiated}, i.e., belong to $E$, and on which value $\xval{E}$ the variable $\var{E}$ for these instantiated nodes assumes.
\par 
It then follows from Eqs.~\eqref{eq:backwards-global} and~\eqref{eq:function-branching} and the strong factorisation property\footnote{\label{fn:doesntmatter}This, together with the course of reasoning leading to Eq.~\eqref{eq:algorithm-in-target}, shows that the results of updating the tree (and the algorithm we are deriving) in this way will be exactly the same \emph{for any way} of forming a product of the local models for the children of $s$, \emph{provided only that this product is strongly factorising}. For instance, replacing the conditionally independent natural extension with the strong product in Eq.~\eqref{eq:backwards-nine} will lead to exactly the same inferences. Of course, this should not be taken to mean that our algorithm also works for updating credal trees under strong independence.} that [see the Appendix for a proof]
\begin{equation}\label{eq:algorithm-off-backbone}
  \lglobconm[\phi^\mu_s]{s}=\lmess{s}
  \text{ and }
  \uglobconm[\phi^\mu_s]{s}=\umess{s}. 
\end{equation}

\subsubsection*{Next, we turn to nodes $s\precedes t$.} 
Define the messages $\mess{s}{\mu}$ by
\begin{equation}\label{eq:messages-on-backbone-1}
  \mess{s}{\mu}\coloneqq\llocconm[\psi^\mu_s]{s},
\end{equation}
where the gambles $\psi^\mu_s$ on $\values{s}$ are given by the recursion relations:
\begin{equation}\label{eq:messages-on-backbone-2}
  \psi^\mu_t
  \coloneqq\max\{g-\mu,0\}\smashoperator{\prod_{c\in\children{t}}}\lmess{c}
  +\min\{g-\mu,0\}\smashoperator{\prod_{c\in\children{t}}}\umess{c},
\end{equation}
and for each $\init\neq s\precedes t$, so $\mother{s}$ exists,
\begin{equation}\label{eq:messages-on-backbone-3}
  \psi^\mu_{\mother{s}}
  \coloneqq
  \bigg[
    \max\{\mess{s}{\mu},0\}\smashoperator{\prod_{c\in\siblings{s}}}\lmess{c}
    +\min\{\mess{s}{\mu},0\}\smashoperator{\prod_{c\in\siblings{s}}}\umess{c}
  \bigg]
  g^\mu_{\mother{s}}.
\end{equation}
The messages $\mess{s}{\mu}$ are again tuples of real numbers, with one component $\mess{s}{\mu}(\xval{\mother{s}})$ for each of the possible values $\xval{\mother{s}}$ of $\var{\mother{s}}$.\footnote{If $s$ is the root node, then $\mother{s}=\emptyset$ and $\mess{s}{\mu}$ is a single real number, which by Eq.~\eqref{eq:algorithm-in-target} is equal to $\rho_g(\mu)$. See also footnote~\ref{fn:empty-product}.} 
They do depend on the choice of $g$ or $\mu$, as well as on which nodes are instantiated and on which value $\xval{E}$ the variable $\var{E}$ for these instantiated nodes assumes. 
\par
It then follows from Eqs.~\eqref{eq:backwards-global} and~\eqref{eq:function-branching} and the strong factorisation property of the local independent products that [see the Appendix for a proof] 
\begin{equation}\label{eq:algorithm-in-target}
  \lglobconm[\phi^\mu_s]{s}=\mess{s}{\mu}
  \text{ and of course }\rho_g(\mu)=\mess{\init}{\mu}.
\end{equation}
We conclude that we can find the value of $\rho_g(\mu)$ by a backwards recursion method consisting in passing messages up to the root of the tree, and in transforming them in each node using the local uncertainty models; see Eqs.~\eqref{eq:lower-upper-messages} and~\eqref{eq:messages-on-backbone-1}--\eqref{eq:messages-on-backbone-3}.
\par
There is a further simplification, because we are not necessarily interested in the actual value of $\rho_g(\mu)$, but rather in its sign. 
It arises whenever there are instantiated nodes above the target node: $E\cap\ancestors{t}\neq\emptyset$. 
Let in that case $e_t$ be the greatest element of the chain $E\cap\ancestors{t}$, i.e., the instantiated node closest to and preceding the target node $t$, and let $s_t$ be its successor in the chain $\until{t}$; see for instance Fig.~\ref{fig:mtree}. 
If we let
\begin{equation*}
   \lambda_g(\mu)
   \coloneqq\max\{\mess{s_t}{\mu}(\xval{e_t}),0\}
  \smashoperator{\prod_{c\in\siblings{s_t}}}\lmess{c}(\xval{e_t})
  +\min\{\mess{s_t}{\mu}(\xval{e_t}),0\}
  \smashoperator{\prod_{c\in\siblings{s_t}}}\umess{c}(\xval{e_t}),
\end{equation*}
then it follows from Eq.~\eqref{eq:messages-on-backbone-3} [with
$s=s_t$ and $\mother{s}=e_t$] that $\psi^\mu_{e_t}=\xindic{e_t}\lambda_g(\mu)$. 
If we now continue to use Eqs.~\eqref{eq:messages-on-backbone-2} and~\eqref{eq:messages-on-backbone-3} until we reach the root of the tree, we eventually find that\footnote{Actually, we easily derive that $\rho_g(\mu)
=a\max\{\lambda_g(\mu),0\}+b\min\{\lambda_g(\mu),0\}$, where $a$ and $b$ are real constants that do not depend on $g$ and $\mu$. Letting $g\coloneqq\mu\pm1$ then allows us to identify the  constants $a$ and $b$.} 
\begin{equation}\label{eq:simplification}
  \rho_g(\mu)
  =
  \begin{cases}
    \lpr(\xsing{E})\lambda_g(\mu)
    &\text{ if $\lambda_g(\mu)\geq0$}\\
    \upr(\xsing{E})\lambda_g(\mu)
    &\text{ if $\lambda_g(\mu)\leq0$}.
  \end{cases}
\end{equation}
Since we assumed from the outset that $\upr(\xindic{E})>0$, we gather from Eq.~\eqref{eq:regular} that $\lupcon[g]{t}{E}=\max\set{\mu\in\reals}{\lambda_g(\mu)\geq0}$. 
Moreover, by combining Eqs.~\eqref{eq:function-branching-too} and~\eqref{eq:algorithm-off-backbone} with Proposition~\ref{prop:global:positivity:too}, we find that $\umess{c}(\xval{e_t})=\xinuglobcon[\xsing{E\cap\after{c}}]{c}{e_t}>0$ for all $c\in\siblings{s_t}$, and therefore $\lambda_g(\mu)\geq0\ifonlyif\mess{s_t}{\mu}(\xval{e_t})\geq0$. 
Hence $\lupcon[g]{t}{E}=\max\set{\mu\in\reals}{\mess{s_t}{\mu}(\xval{e_t})\geq0}$.
\par
We conclude that in order to update the tree in the situation described above, we can perform all calculations on the sub-tree $\after{s_t}$, where the new root $s_t$ has local model $\xinlloccon{s_t}{e_t}$. 
This is also borne out by the discussion of the separation properties in Section~\ref{sec:separation}.

\subsection{An algorithm.}
\begin{figure}[hbt]
  \centering
  \begin{tikzpicture}[scale=0.8]\footnotesize
    \matrix [draw=black!50,anchor=west,inner sep=10] at (7,-2) 
    { 
      \node[observed,inner sep=2pt] {$\phantom{x}$}; & \node {: observed node};\\
      \node[target,inner sep=2pt] {$\phantom{x}$};& \node {: queried or target node $t$};\\ 
      \node[nd,inner sep=2pt] {$\phantom{x}$}; & \node {: unobserved node}; \\ 
    };
      \node[nd] (root) at (0,0) { $\init$};
      \node[terminal] (x1)  at ($(root)+(-100:\afstand)$) {$\var{1}$};
      \node[observed,pin=30:$e_t$] (x2)  at ($(root)+( -15:\afstand)$) {$\var{2}$};
      \node[nd,pin=30:$s_t$] (x3)  at ($  (x2)+(-40:\afstand)$) {$\var{3}$};
      \node[nd] (x4)  at ($  (x3)+(-110:\afstand)$) {$\var{4}$};
      
      \node[nd] (x5)  at ($  (x4)+(-140:\afstand)$) {$\var{5}$};
      \node[observed] (x15)  at ($  (x4)+(170:\afstand)$) {$\var{15}$};
      \node[observed] (x6)  at ($  (x5)+(-140:\afstand)$) {$\var{6}$};
      \node[observed] (x7)  at ($  (x5)+(-80:\afstand)$) {$\var{7}$};
      \node[observed] (x8)  at ($  (x7)+(-45:\afstand)$) {$\var{8}$};
      \node[observed] (x9)  at ($  (x7)+(-100:\afstand)$) {$\var{9}$};
      \node[terminal] (x16)  at ($  (x6)+(-150:\afstand)$) {$\var{16}$};

     \node[target,pin=30:$t$] (x10) at ($  (x4)+(-40:\afstand)$) {\textcolor{white}{$\var{10}$}};
      \node[observed] (x11) at ($ (x10)+(-100:\afstand)$) {$\var{11}$};
      \node[nd] (x12) at ($ (x10)+(-35:\afstand)$) {$\var{12}$};
      \node[observed] (x14) at ($ (x12)+(-35:\afstand)$) {$\var{14}$};
      \node[terminal] (x13) at ($ (x12)+(-100:\afstand)$) {$\var{13}$};
      \draw[ed] (root) -- (x1);
      \draw[ed] (root) -- (x2);
      \draw[ed] (x2) -- (x3);
      \draw[backbone] (x3) -- (x4);
      \draw[ed] (x4) -- (x5);
      \draw[ed] (x5) -- (x6);
      \draw[ed] (x5) -- (x7);
      \draw[ed] (x7) -- (x8);
      \draw[ed] (x7) -- (x9);
      \draw[ed] (x6) -- (x16);
      \draw[ed] (x4) -- (x15);
      \draw[backbone] (x4) -- (x10);
      \draw[ed] (x10) -- (x11);
      \draw[ed] (x10) -- (x12);
      \draw[ed] (x12) -- (x13);
      \draw[ed] (x12) -- (x14);
\end{tikzpicture}
\caption{Example imprecise Markov tree. The target node is $t=10$, $e_t=2$ is the `greatest' observed ancestor of $t$ and $s_t=3$ is the child of $e_t$ that precedes $t$. The bolder arrows represent the trunk $\trunk=\{3,4,10\}$ of the tree.}
  \label{fig:mtree}
\end{figure}
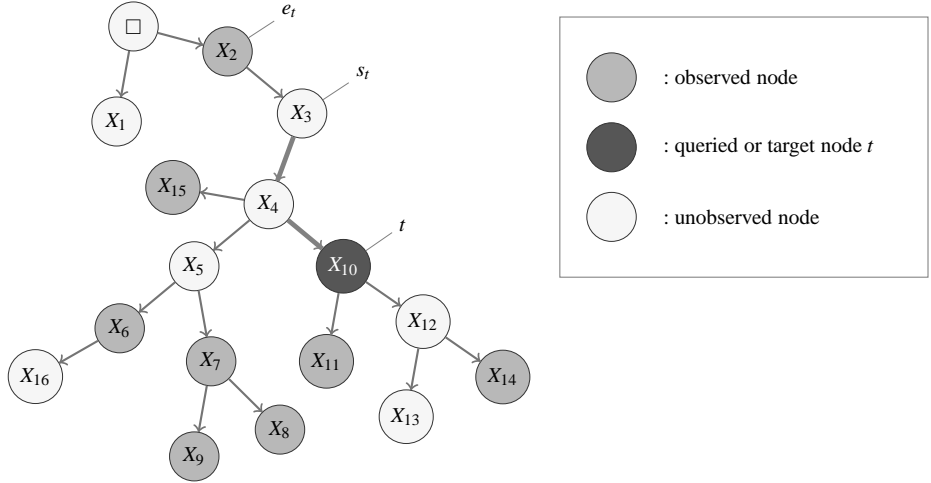
We now convert these observations into a workable algorithm. 
\par
Using regular extension and message passing, we are able to compute $\lupcon[g]{t}{E}$: we
\begin{inparaenum}[(i)]
\item choose any $\mu\in[\min{g},\max{g}]$;
\item calculate the value of $\lambda_g(\mu)$ by sending messages from the terminal nodes towards the root; and
\item repeat this in some clever way to find the maximal $\mu$ that will make this $\lambda_g(\mu)$ zero.
\end{inparaenum}
But we have seen above that this naive approach can be sped up by exploiting
\begin{inparaenum}[(a)]
\item the separation properties of the tree, and
\item the independence of $\mu$ (and $g$) for some of the messages, namely those associated with nodes that do not precede the target node $t$.
\end{inparaenum}
\par
For a start, as we are only interested in the sign of $ \rho_g(\mu)$ [or equivalently, that of $\lambda_g(\mu)$], which we have seen is determined by the sign of $\mess{s_t}{\mu}(\xval{e_t})$, we only have to take into consideration nodes that strictly follow $e_t$.
\par
The next thing a smarter implementation of the algorithm can do, is determine the \emph{trunk} $\trunk$ of the tree:  those nodes that precede the queried node $t$ and strictly follow the greatest observed node $e_t$ preceding $t$. We can define the trunk more formally as follows: $\trunk\coloneqq\until{t}\cap\after{\children{e_t}}$. For the tree in Fig.~\ref{fig:mtree} for instance, where the darker $\var{10}$ is the queried variable and the lighter nodes $\{2,6,7,8,9,11,14,15\}$ are instantiated, the trunk is given by $\trunk=\{3,4,10\}$, and indicated by bolder arrows.

\begin{figure}[ht]
  \centering
  \begin{tikzpicture}\footnotesize
      \node[nd] (x4s)  at (2,-10) {$\var{4}$};
      \node[rounded corners, fill=blue!10,draw=blue] (Pi) at ($(x4s)+(30:\afstand)$) 
      {$\luMess{4}=\lumess{5}\lumess{14}$};
      \node[nd] (x5s)  at ($  (x4s)+(-140:\afstand)$) {$\var{5}$};
      \node[observed] (x14s)  at ($  (x4s)+(170:\afstand)$) {$\var{14}$};
      \node[observed] (x6s)  at ($  (x5s)+(-140:\afstand)$) {$\var{6}$};
      \node[observed] (x7s)  at ($  (x5s)+(-80:\afstand)$) {$\var{7}$};
      \node[observed] (x8s)  at ($  (x7s)+(-45:\afstand)$) {$\var{8}$};
      \node[observed] (x9s)  at ($  (x7s)+(-100:\afstand)$) {$\var{9}$};
      \node[terminal] (x16s)  at ($  (x6s)+(-150:\afstand)$) {$\var{16}$};
      \draw[ed] (x4s) -- (x5s);
      \draw[ed] (x4s) -- (x14s);
      \draw[ed] (x5s) -- (x6s);
      \draw[ed] (x5s) -- (x7s);
      \draw[ed] (x7s) -- (x8s);
      \draw[ed] (x7s) -- (x9s);
      \draw[ed] (x6s) -- (x16s);

      \printMessage{x5s}{x4s}{left};
      \node[anchor=north west]  at ($ (m)!0.1!-90:(x5s) $) 
      {$\lumess{5}=\luloccon[\lumess{6}\lumess{7}]{5}{4}$};
      \printMessage{x6s}{x5s}{right};
      \node[anchor=south east]  at ($ (m)!0.1!90:(x6s) $) 
      {$\luloccon[\sing{\xval{6}}]{6}{5}=\lumess{6}$};
      \printMessage{x7s}{x5s}{left};
      \node[anchor=west]  at ($ (m)!0.0!90:(x7s) $) 
      {$\lumess{7}=\luloccon[\sing{\xval{7}}]{7}{5}\lumess{8}\lumess{9}$};
      \printMessage{x8s}{x7s}{left};
      \node[anchor=west]  at ($ (m)!0.1!90:(x8s) $) 
      {$\lumess{8}=\luloccon[\sing{\xval{8}}]{8}{7}$};
      \printMessage{x9s}{x7s}{right};
      \node[anchor=east]  at ($ (m)!0.0!90:(x9s) $) 
      {$\lumess{9}=\luloccon[\sing{\xval{9}}]{9}{7}$};
      \printMessage{x14s}{x4s}{right};
      \node[anchor=south]  at ($ (m)!0.0!90:(x14s) $) 
      {$\lumess{14}$};
      \printMessage{x16s}{x6s}{right};
      \node[anchor=south east]  at ($ (m)!0.1!90:(x16s) $) 
      {$1=\lumess{16}$};
            
      \draw[message] (x4s) -- (Pi.south west);
  \end{tikzpicture}
  \caption{Calculation of $\luMess{4}$, which is a summary of the $\mu$-independent messages in the trunk node $4$.}
  \label{fig:algorithm}
\end{figure}
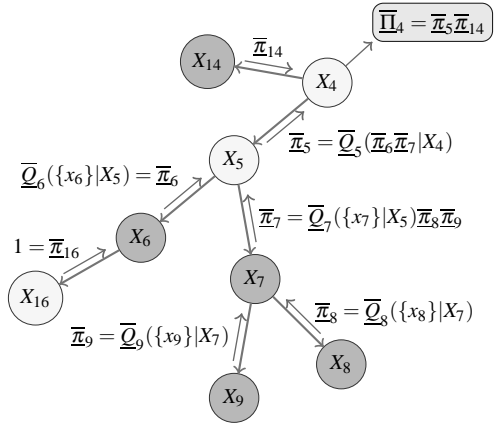
\par 
We have a special interest in the nodes that constitute the trunk, because only they will send messages to their mother nodes that actually depend on $\mu$. 
As a consequence, all other nodes (all descendants of the trunk that are not in the trunk themselves) send messages that have to be calculated only once. 
This implies that we can summarise all the $\mu$-independent messages by propagating all of them until they reach the trunk. 
The $\mu$-independent messages $\lumess{s}$ that arrive in a trunk node $s$ can be represented more succinctly by their point-wise products $\luMess{s}\coloneqq\prod_{c\in{\children{s}}\setminus\trunk}\lumess{c}$,  because Eqs.~\eqref{eq:messages-on-backbone-2} and~\eqref{eq:messages-on-backbone-3} only depend on them through on these products.
\par
This means that for every trunk node $s\in\trunk$, we have to find the lower (upper) messages of every child $c$ of $s$ that is not in the trunk itself. 
Both $\lmess{c}$ and $\umess{c}$ can be calculated recursively using Eq.~\eqref{eq:algorithm-off-backbone}, where the recursion starts at the leaves and moves up to (but stops right before) the trunk. 
In the leaves, the local lower and upper previsions of the indicator of the evidence are sent upwards if the leaf is instantiated; if not the constant 1 is sent up, which is equivalent to deleting the node from the tree. 
We could envisage removing \emph{barren nodes} (all of whose descendants are uninstantiated, such as $\var{1}$, $\var{13}$, $\var{16}$ in the example tree above) from the tree beforehand, but we believe the computational overhead created by the search for them will void the gain. 
\par
\begin{figure}[hbt]
  \centering
  \begin{tikzpicture}\footnotesize
     \node[observed] (x2b)  at (4,2) {$\xval{e_t}$};
      \node[nd] (x3b)  at ($  (x2b)+(-40:\afstand)$) {$\var{s_t}$};
      \node[nd] (x4b)  at ($  (x3b)+(-90:\afstand)$) {$\var{4}$};
      \node[target] (x10b) at ($  (x4b)+(-90:\afstand)$) {\textcolor{white}{$\var{t}$}};
      \draw[backbone] (x3b) -- (x4b);
      \draw[backbone] (x4b) -- (x10b);
      \printMessage{x3b}{x2b}{left};
      \node[anchor=west]  at ($ (m)!0.1!-90:(x3b) $) 
      {$\mess{s_t}{\mu}(\xval{e_t})
        =\xinlloccon[\underbrace{\max\{\mess{s_t}{\mu},0\}\lMess{s_t}
          +\min\{\mess{s_t}{\mu},0\}\uMess{s_t}}_{\psi^\mu_{s_t}}]{s_t}{e_t}$};
      \printMessage{x4b}{x3b}{left};
      \node[anchor=north west,yshift=7]  at ($ (m)!0.1!90:(x4b) $) 
      {$\mess{4}{\mu}=\lloccon[\underbrace{\max\{\mess{t}{\mu},0\}\lMess{4} 
          +\min\{\mess{t}{\mu},0\}\uMess{4}}_{\psi^\mu_{4}}]{4}{s_t}$};
      \printMessage{x10b}{x4b}{left};
      \node[anchor=north west,yshift=7]  at ($ (m)!0.1!90:(x10b) $) 
      {$\mess{t}{\mu}=\lloccon[\underbrace{\max\{g-\mu,0\}\lMess{t} 
          +\min\{g-\mu,0\}\uMess{t}}_{\psi^\mu_t}]{t}{4}$};
      \end{tikzpicture}
  \caption{Calculation of $\mess{s_t}{\mu}(x_{e_t})$, whose sign is the same as that of $\lpr(\xindic{E}[g-\mu])$.}
  \label{fig:joint}
\end{figure}
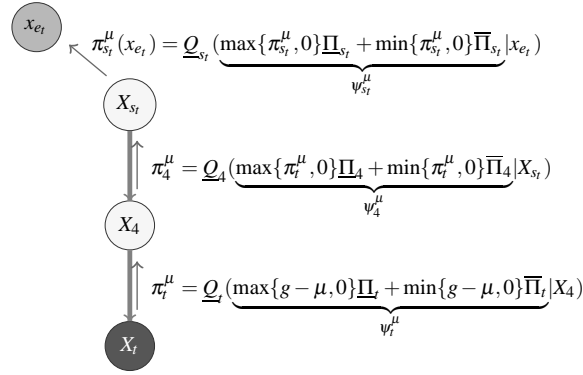
\par
The only recursion that is still left to do, is the calculation of the $\mu$-dependent messages $\mess {s}{\mu}$ along the trunk.
As demonstrated in Fig.~\ref{fig:joint}, we can calculate $\mess{s_t}{\mu}(e_t)$ using the following recursion formula: 
\begin{equation*}
  \mess{s}{\mu} \coloneqq
  \begin{cases}
    \llocconm[\max\{g-\mu,0\}\lMess{s}+\min\{g-\mu,0\}\uMess{s}]{s}
    &s=t,\\
    \llocconm[\max\{\mess{c_t}{\mu},0\}\lMess{s}+\min\{\mess{c_t}{\mu},0\}\uMess{s}]{s}
    &s\in\trunk\setminus\{t\}\text{ and }\children{s}\cap\trunk=\sing{c_t}.
  \end{cases}
\end{equation*} 
These formulas are reformulations of Eqs.~\eqref{eq:messages-on-backbone-1}--\eqref{eq:messages-on-backbone-3}, where the influence of the $\luMess{}$ has been made explicit.
\par
Since we now know how to calculate $\mess{s_t}{\mu}(e_t)$, we can tackle the final problem: find the maximal $\mu$ for which $\mess{s_t}{\mu}(e_t)=0$. 
In principle, a secant root-finding method could be used, but using the concavity and non-increasing character of $\mess{s_t}{\mu}(e_t)$ as a function of $\mu$, we can speed up the calculation of the maximal root drastically as shown in Fig.~\ref{fig:rootfinding}.

\begin{figure}
\centering
  \begin{tikzpicture}[xscale=4.0,yscale=4.0,domain=0.1:1]\small
    \draw[->,gray,very thin] (-0.1,0) -- (1.5,0) node[right] {$\mu$};
    \draw[->,gray,very thin] (-0.08,-0.4) -- (-0.08,0.5) node[left] {$\rho_g(\mu)$};
    \draw[color=black,thick] plot (\x,{0.4-exp(4*\x)/60});
    \draw[color=gray,thin] (0.1,{0.4-exp(4*0.1)/60}) -- (0.65,{0.4-exp(4*0.65)/60}) -- ({((24-exp(4*0.65))*0.1-(24-exp(4*0.1))*0.65)/(exp(4*0.1)-exp(4*0.65))},0);
    \draw[color=gray,thin] (1.0,{0.4-exp(4*1.0)/60}) -- (0.95,{0.4-exp(4*0.95)/60}) -- ({((24-exp(4))*0.95-(24-exp(4*0.95)))/(exp(4*0.95)-exp(4))},0);
    \draw[color=gray,thin] (0.65,{0.4-exp(4*0.65)/60}) -- ({((24-exp(4*0.65))*0.95-(24-exp(4*0.95))*0.65)/(exp(4*0.95)-exp(4*0.65))},0) -- (0.95,{0.4-exp(4*0.95)/60});
    \fill[red] ({ln(24)/4},0) circle (0.3pt) node[above,xshift=4pt] {$t$};
    \fill[black] (0.1,{0.4-exp(4*0.1)/60}) circle (0.3pt) node [above] {$(a,\rho_g(a))$};
    \fill[black] (0.65,{0.4-exp(4*0.65)/60}) circle (0.3pt) node [above right] {$(b,\rho_g(b))$};
    \fill[black] (0.95,{0.4-exp(4*0.95)/60}) circle (0.3pt) node [right,yshift=2pt] {$(c,\rho_g(c))$};
    \fill[black] (1.0,{0.4-exp(4*1.0)/60}) circle (0.3pt) node [right,yshift=-0pt] {$(d,\rho_g(d))$};
    \fill[blue] ({((24-exp(4*0.65))*0.95-(24-exp(4*0.95))*0.65)/(exp(4*0.95)-exp(4*0.65))},0) circle (0.2pt) 
    node [below left] {$p$};
    \fill[blue] ({((24-exp(4))*0.95-(24-exp(4*0.95)))/(exp(4*0.95)-exp(4))},0) circle (0.2pt) 
    node [below right,xshift=+4pt] {$q$};
    \fill[blue] ({((24-exp(4*0.65))*0.1-(24-exp(4*0.1))*0.65)/(exp(4*0.1)-exp(4*0.65))},0) circle (0.2pt) 
    node [below right] {$r$};
    \node[rectangle] at (2.5,0) {
      \begin{minipage}[ht]{0.5\linewidth}
        $p\coloneqq\nicefrac{\rho_g(c) b-\rho_g(b) c}{\rho_g(c)-\rho_g(b)}$;\\
        $m\coloneqq c$;\;$t\coloneqq\nicefrac{p+m}{2}$;\\
        {\bf while } $m-p>\mathit{tol}$ { \bf and }$\rho_g(t)\neq0.0$\\
        $\quad\text{\bf if }\rho_g(t)>0$\\
        $\qquad a\coloneqq b;\; b\coloneqq t$;\\
        $\qquad s\coloneqq\nicefrac{\rho_g(a)b-\rho_g(b)a}{\rho_g(a)-\rho_g(b)}$;\\
        $\quad\text{\bf else}$\\
        $\qquad d\coloneqq c;\; c\coloneqq t$;\\
        $\qquad s\coloneqq\nicefrac{\rho_g(c) d-\rho_g(d) c}{\rho_g(c)-\rho_g(d)}$;\\
        $\quad p\coloneqq\nicefrac{\rho_g(c)b-\rho_g(b)c}{\rho_g(c)-\rho_g(b)}$;\\
        $\quad m\coloneqq\min\{m,s\}$;\;$t\coloneqq\nicefrac{p+m}{2}$;
      \end{minipage}
    };
  \end{tikzpicture}
\caption{The root of a concave and non-increasing function $\rho_g$ whose values $\rho_g(a)>\rho_g(b)>0>\rho_g(c)>\rho_g(d)$ are known, will always be in the interval $[p,m]$ with $m\coloneqq\min\{q,r\}$. Here $p,q$ and $r$ are the intersections with the horizontal axis of the straight lines through $(b,\rho_g(b))$ and $(c,\rho_g(c))$, $(c,\rho_g(c))$ and $(d,\rho_g(d))$, and $(a,\rho_g(a))$ and $(b,\rho_g(b))$, respectively. The next function evaluation of $\rho_g$ will be in $t$ which bisects the error interval $[p,m]$. If $\rho_g(t)>0$, then $a$ becomes $b$ and $b$ becomes $t$, otherwise $d$ becomes $c$ and $c$ becomes $t$ and a new interval $[p,m]$ and matching $t$ can be calculated. We stop iterating as soon as the error interval $[p,m]$ is smaller than a given tolerance $\mathit{tol}$, or $\rho_g(t)$ is exactly zero.}
  \label{fig:rootfinding}
\end{figure}
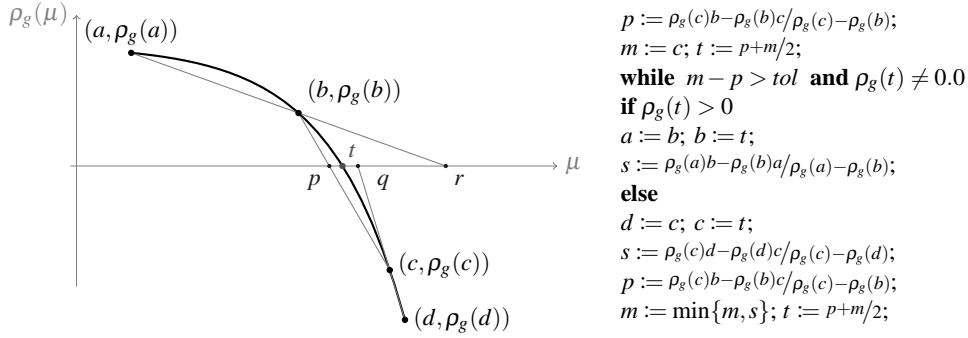
\par
Let us briefly discuss the complexity of our algorithm.
Consider for a start that for a fixed $\mu$ each node makes a single local computation and then propagates the result to its mother node: this implies that, with $\mu$ fixed, the algorithm is linear in the number of nodes.
Iterating on $\mu$ then amounts to multiplying such a linear complexity with the number of iterations. 
This number depends on the function $g$, as the iterations are made to compute the root of a function that is known to belong to the real interval $[\min g,\max g]$. 
If we assume that the bisection algorithm is employed to find the root---for the sake of simplicity---and let $r\coloneqq\max g-\min g$ be the range of the function, then the number of iterations is bounded by $\log_2\frac{r}{\mathit{tol}}+1$, where $\mathit{tol}$ is some fixed tolerance. In other words, the number of iterations is linear in the number $b$ of bits needed to represent $r$ in base 2. 
This means that the overall complexity of the algorithm is $O(b\cdot|T|)$, taking into account that the computational complexity of our root-finding algorithm must be lower than for the bisection (and actually also for the secant) algorithm. 
Since $b$ will be a small number\footnote{It could be argued that $b$ should be bounded given the finiteness of a computer's way to represent numbers.} in most cases (e.g.~when the focus is on probabilities), we simply refer to the complexity of our algorithm as linear in the number of nodes.

\section{A simple example involving dilation}
\label{sec:example}
We present a very simple example that allows us to
\begin{inparaenum}[(i)]
\item follow the inference method discussed above in a step-by-step fashion;
\item see that there are separation properties for credal nets under strong independence that fail for credal trees under epistemic irrelevance; and
\item see that in that case we will typically observe dilation.
\end{inparaenum}
\par
Consider the following imprecise Markov chain:
\begin{center}
  \begin{tikzpicture}\small
    \tikzstyle{edge from parent}+=[->,semithick]
    \tikzstyle{infoline}=[->,thick,dotted,draw=blue!50]
    \node (x1) {$\var{1}$} 
    [grow=right,level distance=35]
    child {node (x2) {$\var{2}$}
    child {node (x3) {$\var{3}$}}};
    \node[below of=x1,node distance=8mm,color=blue] (x1info) {?};
    \node[below of=x2,node distance=8mm,color=blue] (x2info) {$\xval{2}$}; 
    \node[below of=x3,node distance=8mm,color=blue] (x3info) {$\xval{3}$}; 
    \path[infoline] (x1) -- (x1info);
    \path[infoline] (x2info) -- (x2);
    \path[infoline] (x3info) -- (x3);
  \end{tikzpicture}
\end{center}
To make things as simple as possible, we suppose that $\values{1}=\{a,b\}$ and that $\lloc{1}$ is a linear (or precise, or expectation-like) model $\loc{1}$ with mass function $\mass$. 
We also assume that $\lloccon{2}{1}$ is a linear model $\loccon{2}{1}$ with conditional mass function $\masscon{1}$. 
We make no such restrictions on the local model $\lloccon{3}{2}$. 
We also use the following simplifying notational device: if we have three real numbers $\underline{\kappa}$, $\overline{\kappa}$ and $\gamma$, we let
\begin{equation*}
  \sprod{\kappa}{\gamma}
  \coloneqq\underline{\kappa}\max\{\gamma,0\}+\overline{\kappa}\min\{\gamma,0\}.
\end{equation*}
We observe $\var{2}=\xval{2}$ and $\var{3}=\xval{3}$, and want to make inferences about the target variable $\var{1}$: for any $g\in\gambles{1}$, we want to know $\lupcon[g]{1}{\{2,3\}}$. 
Letting $\ltarget\coloneqq\lupcon[\{a\}]{1}{\{2,3\}}$ and $\utarget\coloneqq\uupcon[\{a\}]{1}{\{2,3\}}$, we infer from the separate coherence of $\lupcon{1}{\sing{2,3}}$ that it suffices to calculate $\ltarget$ and $\utarget$, because
\begin{equation*}
  \lupcon[g]{1}{\{2,3\}}=g(b)+\sprod{\target}{g(a)-g(b)}.
\end{equation*}
We let $g^\mu=[\ind{\sing{a}}-\mu]\xindic{2}\xindic{3}$, and apply the approach of the previous section. 
We see that the trunk $\tilde{\nodes}=\{1\}$, and the instantiated leaf node $3$ sends up the messages $\lumess{3}=\luloccon[\xsing{3}]{3}{2}$ to the instantiated node $2$, which transforms them into the messages 
\begin{equation*}
  \lumess{2}
  =\luloccon[\xsing{2}]{2}{1}\lumess{3}(\xval{2})
  \eqqcolon\masscon[\xval{2}]{1}\lumass,
\end{equation*}
where we let $\masscon[\xval{2}]{1}\coloneqq\luloccon[\xsing{2}]{2}{1}$ and $\lumass\coloneqq \lumess{3}(\xval{2})$. 
These messages are sent up to the (target) root node $t=1$, which transforms them into the message $\mess{1}{\mu}=\loc{1}(\psi^\mu_1)$ with $\psi^\mu_1=\masscon[\xval{2}]{1}\sprod{\mass}{\ind{\sing{a}}-\mu}$. 
If we also use that $0\leq\mu\leq 1$, this leads to
\begin{equation*}
  \lglob{1}(g^\mu)
  =\mess{1}{\mu}
  =\mass(a)\mass(\xval{2}\vert a)\lmass[1-\mu]
  +\mass(b)\mass(\xval{2}\vert b)\umass[-\mu],
\end{equation*}
so we find after applying regular extension that
\begin{align*}
  \ltarget
  &=\lupcon[\{a\}]{1}{\{2,3\}}
  =\frac{\mass(a)\mass(\xval{2}\vert a)\lmass}
  {\mass(a)\mass(\xval{2}\vert a)\lmass
    +\mass(b)\mass(\xval{2}\vert b)\umass}\\
  \utarget
  &=\uupcon[\{a\}]{1}{\{2,3\}}
  =\frac{\mass(a)\mass(\xval{2}\vert a)\umass}
  {\mass(a)\mass(\xval{2}\vert a)\umass
    +\mass(b)\mass(\xval{2}\vert b)\lmass}.
\end{align*}
When $\lmass=\umass$, which happens for instance if the local model for $\var{3}$ is precise, then we see that, with obvious notations,
\begin{equation}\label{eq:precise}
  \utarget
  =\ltarget
  =\frac{\mass(a)\mass(\xval{2}\vert a)}
  {\mass(a)\mass(\xval{2}\vert a)
    +\mass(b)\mass(\xval{2}\vert b)}
  \eqqcolon p(a\vert\xval{2})
\end{equation}
and therefore $\var{2}$ indeed separates $\var{3}$ from $\var{1}$. 
But in general, letting $\alpha\coloneqq\mass(a)\mass(\xval{2}\vert a)$ and $\beta\coloneqq\mass(b)\mass(\xval{2}\vert b)$, we get 
\begin{equation*}
 \utarget-p(a\vert\xval{2})  
  =\frac{\alpha\beta}{\alpha+\beta}
  \frac{\umass-\lmass}
  {\alpha\umass+\beta\lmass}
  \text{ and }
  p(a\vert\xval{2})
  -\ltarget
  =\frac{\alpha\beta}{\alpha+\beta}
  \frac{\umass-\lmass}
  {\alpha\lmass+\beta\umass}.
\end{equation*}
As soon as $\umass>\lmass$, $\var{2}$ no longer separates $\var{3}$ from $\var{1}$, and we witness \emph{dilation} \cite{herron1997,seidenfeld1993} because of the additional observation of $\var{3}$!

\section{Numerical comparison with strong independence}
\label{sec:comparison}
Strong independence implies epistemic irrelevance, and hence inferred (lower-upper) probability intervals for imprecise Markov trees with epistemic irrelevance will include those obtained assuming strong independence. 
This suggests that our algorithm could also be used also as a tool to make conservative (also called outer) approximations of the computations made in a credal tree under strong independence. 
This could be an important application of our algorithm since at the moment it is unclear whether or not updating probabilities in a tree is a polynomial task under strong independence. 
If it were not, addressing the problem would definitely benefit from the availability of fast approximations.

With this idea in mind, we make a preliminary empirical exploration of the quality of the approximation. 
As noted in Section~\ref{sec:separation}, the two models have different separation properties: this is particularly important when evidence is back-propagated from leaves to root.  
For this reason, we compare posterior probability intervals for the root variable of a \emph{chain} where only the leaf node is instantiated.
\par
Fig.~\ref{fig:strvsepi} reports the results of this comparison for chains with binary nodes, randomly generated local models, and variable length (from 5 up to 100 nodes). 
The algorithm in Section~\ref{sec:algorithm} has been used to compute the posterior probability intervals in the chains under epistemic irrelevance, while the \emph{2U algorithm} \cite{fagiuoli1998} was used for updating in the chains under strong independence. 
The inferred probability intervals for the former turn out to be clearly wider, and the mean difference between the two intervals is about $0.3$ irrespective of the length of the chain, at least for chains with more than ten nodes. 
\par
For non-binary nodes there are no efficient algorithms known for updating chains with strong independence. 
We used the procedure in \cite{campos2005} to update chains with less than seven ternary nodes and credal sets with three randomly generated extreme points in the strong independence case. 
A similar difference between the posterior intervals was observed also in these cases. 
For longer chains, updating for the chain under strong independence is too slow and no comparison can be made.
In summary, there is a non-negligible difference between inferences based on the two notions of `independence'. This means that the epistemic approximations to the strong case could be quite crude in practise. 
However, their being outer (that is, safe) approximations together with their light complexity could still make of them very useful tools, whenever the strong independence approach is deemed necessary or appropriate.
\par
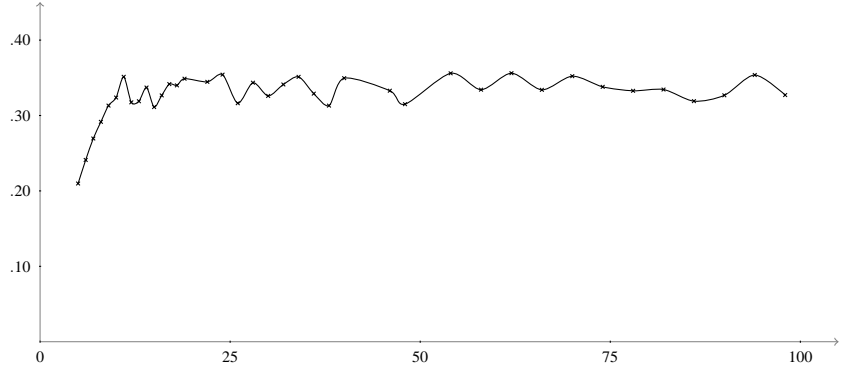
\begin{figure}[hbt!]
  \centering
  \begin{tikzpicture}[xscale=1.0/10.0,yscale=1.0/10.0]\small
    \draw[->,gray,very thin] (0.0,0) -- (105,0) node[right] {};
    \draw[->,gray,very thin] (0.0,0.0) -- (0.0,45) node[left] {};
    \foreach \x/\xtext in {0/0, 25/25 , 50/50 , 75/75 , 100/100}
    \draw[shift={(\x,0)}] (0pt,4pt) -- (0pt,-4pt) node[below] {\tiny $\xtext$};
    \foreach \y/\ytext in {10/.10, 20/.20, 30/.30 , 40/.40}
    \draw[shift={(0,\y)}] (4pt,0pt) -- (-4pt,0pt) node[left] {\tiny $\ytext$};
    \draw plot[xscale=1.0/1.0,yscale=1.0/1.0,smooth,mark=x,mark size=10pt] file {data.table};
  \end{tikzpicture}
  \caption{Numerical evaluation of the difference between the sizes of the posterior intervals for inferences on credal chains over binary variables with epistemic irrelevance and strong independence. The plot reports the mean difference (in ordinate) as a function of the number of nodes in the chains (in abscissa). Means are estimated over $200$ Monte Carlo runs.}
  \label{fig:strvsepi}
\end{figure}

\section{An application}\label{sec:application}
The tree topology of the graphs considered in this paper is expressive enough to model useful and interesting problems. 
These problems can then be solved efficiently by means of the algorithm described in the previous sections. 
We make this point clearer with an example application about character recognition. 
This is also an opportunity to illustrate the differences between the traditional, precise-probability, approach to the problem and the imprecise-probability one. 
Most notably, these differences arise because the imprecise-probability methods come with the inherent ability to suspend judgement when the information available is deemed insufficient to reliably recognise a character, whereas the precise-probability ones do not.

\subsection{Imprecise hidden Markov models}
Hidden Markov models (HMMs) \cite{rabiner1989} are popular tools for modelling a sequence of hidden variables that generate a related sequence of observable variables. 
These are respectively referred to as the \emph{generative} and \emph{observable} sequences. 
HMMs have applications in many areas of signal processing, and more specifically in speech and text recognition.
\par
Both the generative and the observable sequence are described by sets of variables over the same domain $\values{}$, denoted respectively by $\var{s_1}$, \dots, $\var{s_n}$ and $\var{o_1}$, \dots, $\var{o_n}$. 
The independence assumptions between these variables, which characterise HMMs, are those corresponding to the tree structure below. 
Informally, this topology states that every element of the generative sequence depends only on its predecessor, while each observation depends only on the corresponding element of the generative sequence.
\begin{center}
  \begin{tikzpicture}\small
    \tikzstyle{edge from parent}+=[->,semithick]
    \tikzstyle{link}=[draw,->,semithick]
    \node (x1) {$\var{s_1}$}
    [grow=right,level distance=40]
    child {node (x2) {$\var{s_2}$}
      child {node (x) {$\dots$}
        child {node (xn) {$\var{s_n}$}}}};
    \node[below of=x1,node distance=30] (o1) {$\var{o_1}$};
    \node[below of=x2,node distance=30] (o2) {$\var{o_2}$};
    \node[below of=x,node distance=30] (o) {$\dots$};
    \node[below of=xn,node distance=30] (on) {$\var{o_n}$};
    \path[link] (x1) -- (o1);
    \path[link] (x2) -- (o2);
    \path[link] (xn) -- (on);
   \node[left of=o1,node distance=60] {\emph{observable sequence}:};
    \node[left of=x1,node distance=60] {\emph{generative sequence}:};
  \end{tikzpicture}
\end{center}
A local uncertainty model should be defined for each variable. 
In the case of precise probabilistic assessments, this corresponds to linear (precise, or expectation-like) versions of the local models $\lloc{s_1}$, $\lloccon{s_{k+1}}{s_k}$ and $\lloccon{o_k}{s_k}$, $k=1,\dots,n$, where the conditional models are assumed to be \emph{stationary}, i.e., independent of $k$. 
These model, respectively, beliefs about the first state in the generative sequence, the transitions between adjacent states, and the observation process.
\par
Bayesian techniques for learning from multinomial data are usually employed for identifying these models. 
But, especially if only few data are available, other methods leading to imprecise assessments, such as the \emph{imprecise Dirichlet model} (IDM, \cite{walley1996b}), might offer a more realistic model of the local uncertainty.
For example, for the unconditional local model $\lloc{s_1}$, applying the IDM leads to the following simple identification:
\begin{equation}\label{eq:idm}
  \lloc{s_1}(\{x_1\})=\frac{n^{s_1}_{x_1}}{s+\sum\limits_{x\in\values{}}n^{s_1}_{x}}
  \quad
  \uloc{s_1}(\{x_1\})=\frac{s+n^{s_1}_{x_1}}{s+\sum\limits_{x\in\values{}}n^{s_1}_{x}},
\end{equation}
where $n^{s_1}_{x_1}$ counts the units in the sample for which $\var{s_1}=x_1$, and $s$ is a (positive real) hyper-parameter that expresses the degree of caution in the inferences. 
For the conditional local models, we can proceed similarly. 
This leads to the identification of an \emph{imprecise HMM}, a special credal tree under epistemic irrelevance, like the ones introduced in Section~\ref{sec:markov-trees}.
\par
Generally speaking, the algorithm described in Section~\ref{sec:algorithm} can be used for computing inferences with such imprecise HMMs. 
Below, we address the more specific problem of \emph{on-line recognition}, which consists in the identification of the most likely value of $\var{s_n}$, given the evidence for the whole observational sequence $\var{o_1}=\xval{o_1}$, \dots, $\var{o_n}=\xval{o_n}$. 
For precise local models, this problem requires the computation of the state $\tilde{x}_{s_n}\coloneqq\argmax_{\xval{s_n}\in\values{}}P(\{\xval{s_n}\}\vert\xval{o_1},\dots,\xval{o_n})$  that is most probable after the observation. 
For imprecise local models different criteria can be adopted; see \cite{troffaes2007} for an overview. 
We consider \emph{maximality}: we order the states by $\xval{s_n}>\zval{s_n}$ if and only if $\underline{P}(\xindic{s_n}-\zindic{s_n}\vert\xval{o_1},\dots,\xval{o_n})>0$, and we look for the \emph{undominated} or \emph{maximal} states under this order. 
This may produce \emph{indeterminate} predictions: the set of undominated states may have more than one element.

\subsection{On-line character recognition}
As a very first application of the imprecise HMM, we have considered a \emph{character recognition} problem.\footnote{For a more involved application, related to aircraft trajectory model tracking, see \cite{antonucci2009}.} 
A written text was regarded as a generative sequence, while the observable sequence was obtained by artificially corrupting the text. 
This is a model for a not perfectly reliable observation process, such as the output of an OCR device. 
The local models were identified using the IDM, as in \eqref{eq:idm}, by counting the occurrences of single characters and the `transitions' from one character to another in the generative sequence, and by matchings between the elements of the two sequences. 
By modelling text as a generative sequence, we obviously ignore any dependence there might be between a character and its $n$-th predecessor, for any $n\geq 2$. 
A better, albeit still not completely realistic, model would resort to using $n$-grams (i.e., clusters of $n$ characters with $n\geq2$) instead of monograms. 
Such models might lead to higher accuracy, but they need larger data sets for their quantification, because of the exponentially larger number of possible transitions for which probabilities have to be estimated. 
The figure below depicts how on-line recognition through HMM might apply to this setup.
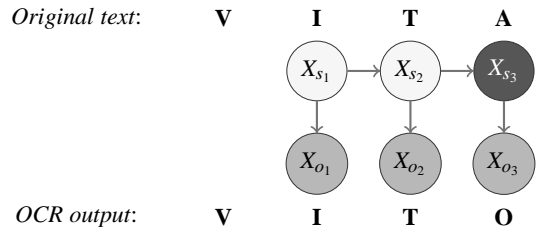
\begin{center}
  \begin{tikzpicture}\small
    \tikzstyle{edge from parent}+=[ed]
   \node[nd] (x1) {$\var{s_1}$}
    [level distance=35]
    child[grow=down] {node[observed] (o1) {$\var{o_1}$}}
    child[grow=right] {node[nd] (x2) {$\var{s_2}$}
      child[grow=down] {node[observed] (o2) {$\var{o_2}$}}
      child[grow=right] {node[target] (x3) {\textcolor{white}{$\var{s_3}$}}
        child[grow=down] {node[observed] (o3) {$\var{o_3}$}}}};
    \node[above of=x1,node distance=20] (xi) {\bfseries I};
    \node[above of=x2,node distance=20] {\bfseries T};
    \node[above of=x3,node distance=20] {\bfseries A};
    \node[left of=xi,node distance=35] (xv) {\bfseries V};
    \node[left of=xv,node distance=55] {\emph{Original text}:};
    \node[below of=o1,node distance=20] (oi) {\bfseries I};
    \node[below of=o2,node distance=20] {\bfseries T};
    \node[below of=o3,node distance=20] {\bfseries O};
    \node[left of=oi,node distance=35] (ov) {\bfseries V};
    \node[left of=ov,node distance=55] {\emph{OCR output}:};
  \end{tikzpicture}
\end{center}
The performance of the precise model can be characterised by its \emph{accuracy} (the percentage of correct predictions) alone. 
The imprecise HMM requires more indicators. 
We follow \cite{corani2008} in using the following:
\begin{description}
\item[determinacy] percentage of determinate predictions,
\item[set-accuracy] percentage of indeterminate predictions containing the right state, 
\item[single accuracy] percentage of correct predictions computed considering only determinate predictions, and \item[indeterminate output size] average number of states returned when the prediction is indeterminate over number of possible states.
\end{description}
\par
\begin{table}[htp!]
  \small
  \begin{tabular}{lrr}
   \bfseries Precise HMM\\ 
   Accuracy & $93.96 \%$ & $(\nicefrac{7275}{7743})$\\
   Accuracy (if imprecise indeterminate)&   $64.97\%$&  $(\nicefrac{243}{374})$\\
   \\
   \bfseries Imprecise HMM    \\
   Determinacy&  $95.17\%$ &   $(\nicefrac{7369}{7743})$\\
   Set-accuracy&           $93.58\%$ &  $(\nicefrac{350}{374})$\\
   Single accuracy&  $95.43\%$ &   $(\nicefrac{7032}{7369})$\\
   Indeterminate output size&   
   $2.97$ out of $21$ classes &  $(\nicefrac{1112}{374})$\\\\
 \end{tabular}
  \caption{Precise vs.~imprecise HMMs. 
Test results obtained by twofold cross-validation on the first two chants of Dante's \emph{Divina Commedia} and $n=2$. 
Quantification is achieved by IDM with $s=2$ and Perks' prior modified as suggested in \cite[Section~5.2]{zaffalon2001}. 
The single-character output by the precise model is then guaranteed to be included in the set of characters the imprecise HMM identifies.}
  \label{tab:results}
\end{table}
\par
The recognition using our algorithm is fast: it never takes more than one second for each character. 
Table~\ref{tab:results} reports descriptive values for a large set ($7743$) of simulations, and a comparison with precise model performance. 
Imprecise HMMs guarantee quite accurate predictions. 
In contrast with the precise model, there are `indeterminate' instances for which they do not output a single state. 
Yet, this happens rarely, and even then we witness a remarkable reduction in the number of undominated states (from the $21$ letters of the Italian alphabet to less than $3$). 
Interestingly, the instances for which the imprecise probability model returns more than one state appear to be `difficult' for the precise probability model: the accuracy of the precise models displays a strong decrease if we focus only on these instances, while the imprecise models here display basically the same performance as for other instances, by returning about three  characters instead of a single one.

\section{Conclusions}
\label{sec:conclusion}
We have defined imprecise-probability (or credal) trees using Walley's notion of epistemic irrelevance. 
Credal trees generalise tree-shaped Bayesian nets in two ways: by allowing the parameters of the tree to be imprecisely specified, and moreover by replacing the notion of stochastic independence with that of epistemic irrelevance. 
Our focusing on epistemic irrelevance is the most original aspect of this work, as this notion has received limited attention so far in the context of credal nets.
\par
We have focused in particular on developing an efficient exact algorithm for updating beliefs on the tree. 
Like the algorithms developed for precise graphical models, our algorithm works in a distributed fashion by passing messages along the tree. 
It computes lower and upper conditional previsions (expectations) with a complexity that is linear in the number of nodes in the tree. 
This is remarkable because until now it was unclear whether an algorithm with the features described above was at all feasible: in fact, epistemic irrelevance is most easily formulated using coherent lower previsions, which have never before been used as such in practical applications of credal nets. 
Moreover, it is at this point not clear that epistemic irrelevance is as `well-behaved' as strong independence is with respect to the graphoid axioms for propagation of probability in graphical models \cite{cozman2005b,moral2005b}.\footnote{Unlike credal nets based on strong independence, a credal net based on epistemic irrelevance cannot generally be seen as equivalent with a set of Bayesian nets \emph{with the same graphical structure}: if it were, then all separation properties of Bayesian nets would simply be inherited, and we have seen in Section~\ref{sec:example} that such is not the case.} 
Our results therefore appear very encouraging, and seem to have the potential to open up new avenues of research in credal nets. 
\par
On a more theoretical side, we have also shown that our credal trees satisfy the important rationality requirement of coherence. 
This has been established under the assumption that the \emph{upper} probability of any possible observation in the tree is positive, which is a very mild requirement. 
The same assumption also allowed us to show that all inferences made by updating the tree will be coherent with each other as well as with the local uncertainty models in the nodes of the tree.
\par
On the applied side, we have presented an application of the credal tree model to the problem of character recognition, where the parameters of the model are inferred from data. 
The empirical results are positive, especially because they show that our credal trees are able to make more reliable predictions than their precise-probability counterparts.
\par
Where to go from here? 
There are many possible avenues for future research. 
\par
It would be very useful to be able to extend the algorithm at least to so-called \emph{polytrees}, which are substantially more expressive graphs than trees are. 
This could be a difficult task to achieve. 
In fact, updating credal nets based on strong independence is an NP-hard task when the graph is more general than a tree \cite{campos2005a}. 
Similar problems might affect the algorithms for credal nets based on irrelevance.
\par
For applications, it would be very important to develop statistical methods specialised for credal nets under irrelevance that avoid introducing excessive imprecision in the process  of inferring probabilities from data. 
This could be achieved, for instance, by using a single global IDM over the variables of the tree rather than many local ones, as we did in our experiments.
\par
Another research direction could be concerned with trying to strengthen the conclusions that epistemic trees lead to. 
There might be cases where our Markov condition based on epistemic irrelevance is too weak as a structural assessment. 
We have discussed situations where this type of Markov condition systematically leads to a dilation of uncertainty when updating beliefs with observations, and indicated that this dilation is related to the (lack of) certain separation properties induced by epistemic irrelevance on a graph. 
Dilation might not be desirable in some applications, and we could be called upon to strengthen the model in order to rule out such behaviour. 
One way to address the issue of dilation---but not necessarily the easiest---could consist in adding additional irrelevance statements to the model, other than those derived from the Markov condition. 
An easier avenue could be based on designing assumptions that together with the Markov condition lead to some stronger separation properties, while not necessarily requiring them to match the common ones used in Bayesian nets.

\section*{Acknowledgements}
Research by De Cooman and Hermans was supported by Flemish BOF project~01107505 and SBO project~060043 of the IWT-Vlaanderen.
Research by Antonucci and Zaffalon has been partially supported by the Swiss NSF grants n.~200020-116674/1 and n.~200020-121785/1. 
This paper has benefited from discussions with Seraf\'{\i}n Moral, Fabio G.~Cozman and Cassio P.~de Campos, and from the generous comments provided by two anonymous referees.


\appendix
\section{Proofs of important results}\label{sec:app}
In this Appendix, we justify formulas~\eqref{eq:algorithm-off-backbone} and~\eqref{eq:algorithm-in-target}, and give proofs for Propositions~\ref{prop:global:positivity} and~\ref{prop:global:positivity:too}, and Theorem~\ref{theo:global-models}. 

\begin{proof}[Proof of Eqs.~\eqref{eq:algorithm-off-backbone} and~\eqref{eq:algorithm-in-target}]
  Let us define the gambles
  \begin{equation*}
    \varmess{s}{\mu}\coloneqq\lglobconm[\phi^\mu_s]{s}\in\gambles{\mother{s}},\quad s\precedes t
  \end{equation*}
  and, with obvious notations,
  \begin{equation*}
    \luvarmess{s}\coloneqq\luglobconm[\phi^\mu_s]{s}\in\gambles{\mother{s}},\quad s\nprecedes t.
  \end{equation*}
  Let the chain $\until{t}$ be given by $\{t_1,\dots,t_r\}$, where $t_1\coloneqq\init$, $t_r\coloneqq t$ and $t_k\coloneqq\mother{t_{k+1}}$ for $k=1,\dots,r-1$.
  If we apply the recursion equation~\eqref{eq:backwards-global} in $s=t_1$ and take into account the separate coherence and the strong factorisation of the conditionally independent natural extension $\linecon{t_1}$, we see that
  \begin{equation}\label{eq:justification:one}
    \varmess{t_1}{\mu}
    =\varmess{\init}{\mu}
    =\rho_g(\mu)
    =\lglob{}(\phi^\mu_{t_1})
    =\lloc{t_1}(\psi^\mu_{t_1}),
  \end{equation}
  where [provided that $t_1\neq t$ and therefore $r>1$]
  \begin{align}
    \psi^\mu_{t_1}
    &\coloneqq\linecon[\phi^\mu_{t_1}]{t_1}
    =\biglinecon[{g^\mu_{t_1}\smashoperator{\prod_{c\in\children{t_1}}}\phi^\mu_c}]{t_1}
    =g^\mu_{t_1}\biglinecon[{\smashoperator[r]{\prod_{c\in\children{t_1}}}\phi^\mu_c}]{t_1}\notag\\
    &=g^\mu_{t_1}\biglinecon[{\lglobcon[\phi^\mu_{t_2}]{t_2}{t_1}
      \smashoperator{\prod_{c\in\children{t_1}\setminus\sing{t_2}}}\phi^\mu_c}]{t_1}
    =g^\mu_{t_1}\biglinecon[{\varmess{t_2}{\mu}
      \smashoperator{\prod_{c\in\children{t_1}\setminus\sing{t_2}}}\phi^\mu_c}]{t_1}\notag\\
    &=\bigg[
    \max\{\varmess{t_2}{\mu},0\}
    \biglinecon[{\smashoperator[r]{\prod_{c\in\children{t_1}\setminus\sing{t_2}}}\phi^\mu_c}]{t_1}
    +\min\{\varmess{t_2}{\mu},0\}
    \biguinecon[{\smashoperator[r]{\prod_{c\in\children{t_1}\setminus\sing{t_2}}}\phi^\mu_c}]{t_1}
    \bigg]g^\mu_{t_1}\notag\\
    &=\bigg[
    \max\{\varmess{t_2}{\mu},0\}
    \smashoperator{\prod_{c\in\children{t_1}\setminus\sing{t_2}}}\lvarmess{c}
    +\min\{\varmess{t_2}{\mu},0\}
    \smashoperator{\prod_{c\in\children{t_1}\setminus\sing{t_2}}}\uvarmess{c}
    \bigg]g^\mu_{t_1}. \label{eq:justification:two}
  \end{align}
  Similarly, we find that
  \begin{equation}\label{eq:justification:three}
    \varmess{t_2}{\mu}
    =\lglobcon[\phi^\mu_{t_2}]{t_2}{t_1}
    =\lloccon[\psi^\mu_{t_2}]{t_2}{t_1},
  \end{equation}
  where [provided that $t_2\neq t$ and therefore $r>2$] in a completely similar way as above
  \begin{equation}\label{eq:justification:four}
    \psi^\mu_{t_2}
    \coloneqq\linecon[\phi^\mu_{t_2}]{t_2}
    =\bigg[
    \max\{\varmess{t_3}{\mu},0\}
    \smashoperator{\prod_{c\in\children{t_2}\setminus\sing{t_3}}}\lvarmess{c}
    +\min\{\varmess{t_3}{\mu},0\}
    \smashoperator{\prod_{c\in\children{t_2}\setminus\sing{t_3}}}\uvarmess{c}
    \bigg]g^\mu_{t_2}.
  \end{equation}
  We can go on in this way until we come to $t_r=t$:
  \begin{equation}\label{eq:justification:five}
    \varmess{t_r}{\mu}
    =\lglobcon[\phi^\mu_{t_r}]{t_r}{t_{r-1}}
    =\lloccon[\psi^\mu_{t_r}]{t_r}{t_{r-1}},
  \end{equation}
  where, using the separate coherence and the strong factorisation of the conditionally independent natural extension $\linecon{t_r}$, 
  \begin{align}
    \psi^\mu_{t}
    =\psi^\mu_{t_r}
    &\coloneqq\linecon[\phi^\mu_{t_r}]{t_r}
    =\biglinecon[{g^\mu_{t_r}\smashoperator{\prod_{c\in\children{t_r}}}\phi^\mu_c}]{t_r}
    =\biglinecon[{[g-\mu]\smashoperator{\prod_{c\in\children{t_r}}}\phi^\mu_c}]{t_r}\notag\\
    &=\max\{g-\mu,0\}\biglinecon[{\smashoperator[r]{\prod_{c\in\children{t_r}}}\phi^\mu_c}]{t_r}
    +\min\{g-\mu,0\}\biguinecon[{\smashoperator[r]{\prod_{c\in\children{t_r}}}\phi^\mu_c}]{t_r}\notag\\
    &=\max\{g-\mu,0\}\smashoperator{\prod_{c\in\children{t}}}\lvarmess{c}
    +\min\{g-\mu,0\}\smashoperator{\prod_{c\in\children{t}}}\uvarmess{c}.\label{eq:justification:six}
  \end{align}
  Clearly, if we can prove that $\luvarmess{s}=\lumess{s}$ for all $s\nprecedes t$, it will follow from the considerations above that also $\varmess{s}{\mu}=\mess{s}{\mu}$ for all $s\precedes t$, and then the proof is complete.
 This is what we now set out to do.
 Consider any $s\nprecedes t$. 
 Then applying the recursion equation~\eqref{eq:backwards-global} and taking into account the separate coherence and the strong factorisation of the conditionally independent natural extension $\linecon{s}$, we see that, provided $s$ is not a terminal node, and with obvious notations,
  \begin{equation}\label{eq:justification:seven}
    \luvarmess{s}
   =\luglobconm[\phi^\mu_{s}]{s}
  =\lulocconm[\overline{\underline{\psi}}_s]{s},
  \end{equation}
  where 
  \begin{align}
    \overline{\underline{\psi}}_{s}
    &\coloneqq\luinecon[\phi^\mu_{s}]{s}
    =\bigluinecon[{g^\mu_{s}\smashoperator{\prod_{c\in\children{s}}}\phi^\mu_c}]{s}
    =g^\mu_{s}\bigluinecon[{\smashoperator[r]{\prod_{c\in\children{s}}}\phi^\mu_c}]{s}\notag\\
   &=g^\mu_s\smashoperator{\prod_{c\in\children{s}}}
   \luglobcon[\phi^\mu_c]{c}{s}
   =g^\mu_s\smashoperator{\prod_{c\in\children{s}}}\luvarmess{c}.\label{eq:justification:eight}
  \end{align}
  If on the other hand $s$ is a terminal node, then we can use Eq.~\eqref{eq:boundary-condition} to find that
  \begin{equation}\label{eq:justification:nine}
    \luvarmess{s}
    =\luglobconm[\phi^\mu_{s}]{s}
    =\lulocconm[\phi^\mu_{s}]{s}
    =\lulocconm[g^\mu_{s}]{s}
    = \lumess{s},
  \end{equation}
  where the last equality follows from Eq.~\eqref{eq:lower-upper-messages}. 
  Now combine Eqs.~\eqref{eq:justification:seven}--\eqref{eq:justification:nine} and use recursion to complete the proof. 
\end{proof}

\begin{proof}[Proof of Proposition~\ref{prop:global:positivity}]
  Fix any $\xval{\nodes}$ in $\values{\nodes}$.
  We need to prove that $\uglob{}(\xsing{T})>0$ and that $\zinuglobcon[\xsing{\after{S}}]{S}{s}>0$ for all $s\in\nonterminals$, non-empty $S\subseteq\children{s}$ and $\zval{s}\in\values{s}$.
  Our argumentation is similar to a special case of the one in Section~\ref{sec:calculating-rho}.
  We use the notation established there, but with in particular $g\coloneqq\mu+1$, $t\coloneqq\init$ and $E\coloneqq\nodes$.
  This implies that $g^\mu=\xindic{T}$, $g^\mu_s=\xindic{s}$ and $\phi^\mu_s=\xindic{\after{s}}$.
  In accordance with Eq.~\eqref{eq:lower-upper-messages}, we define the messages $\umess{s}\in\gambles{\mother{s}}$ and $\uamess{s}\in\gambles{s}$ recursively by:
  \begin{equation}\label{eq:global:positivity}
    \uamess{s}\coloneqq\prod_{c\in\children{s}}\umess{c}
    \text{ and }
    \umess{s}
    \coloneqq\ulocconm[\xindic{s}\uamess{s}]{s}    
    =\uamess{s}(\xval{s})\ulocconm[\xsing{s}]{s},
    \quad s\in\nodes
  \end{equation}
  with, as before by convention $\uamess{s}\coloneqq1$ in all leaves $s$.
  The last equality follows from the separate coherence of the local models $\ulocconm{s}$ and the fact that all messages~$\umess{s}$ and~$\uamess{s}$ are non-negative.
  It is clear from the recursion equations~\eqref{eq:backwards-nine} and~\eqref{eq:backwards-global} [see also Eq.~\eqref{eq:algorithm-off-backbone}, the proof is completely similar to that of Eqs.~\eqref{eq:algorithm-off-backbone} and~\eqref{eq:algorithm-in-target} given above] that $\uglobconm[\xsing{\after{s}}]{s}=\umess{s}$, for all $s\in\nodes$, and that $\uinecon[\xsing{\after{\children{s}}}]{s}=\uamess{s}$ for all $s\in\nonterminals$. 
 Similarly, it follows from Eq.~\eqref{eq:backwards-nine-restrict}, conjugacy and the strong factorisation property of the conditionally independent natural extension that $\uglobcon[\xsing{\after{S}}]{S}{s}=\prod_{c\in S}\umess{c}$ for all $s\in\nonterminals$ and all non-empty $S\subseteq\children{s}$.
 So we have to prove that all values (all components) of all messages $\umess{s}$, $s\in\nodes$ are (strictly) positive.
 This follows at once from the recursion equation~\eqref{eq:global:positivity} and the assumed strict positivity of the local models $\ulocconm{s}$.
\end{proof}

\begin{proof}[Proof of Proposition~\ref{prop:global:positivity:too}]
  Our argumentation is similar to a special case of the one in Section~\ref{sec:calculating-rho}.
  We also use notation similar to that established there, with in particular $g\coloneqq\mu+1$ and $t\coloneqq\init$.
  In accordance with Eq.~\eqref{eq:lower-upper-messages}, we define the messages $\umess{s}\in\gambles{\mother{s}}$  and $\uamess{s}\in\gambles{s}$ recursively by:
\begin{equation}\label{eq:global:positivity:too}
    \uamess{s}\coloneqq\prod_{c\in\children{s}}\umess{c},
    \quad s\in\nodes
  \end{equation}  
  and
  \begin{equation}\label{eq:global:positivity:ttoo}
    \umess{s}\coloneqq
    \begin{cases}
      \uamess{s}(\xval{s})\ulocconm[\xsing{s}]{s}
      &\text{ if $s\in E$}\\
      \ulocconm[\uamess{s}]{s}
      &\text{ if $s\in\nodes\setminus E$}.
    \end{cases}
 \end{equation}
  with, as before by convention $\uamess{s}\coloneqq1$ in all leaves $s$.
  All these messages are non-negative by construction.
  It is clear from the recursion equations~\eqref{eq:backwards-nine} and~\eqref{eq:backwards-global} [see also Eq.~\eqref{eq:algorithm-off-backbone}, the proof is completely similar to that of
Eqs.~\eqref{eq:algorithm-off-backbone} and~\eqref{eq:algorithm-in-target} given above] that $\uglobconm[\xsing{E\cap\after{s}}]{s}=\umess{s}$ for all $s\in\nodes$.
  Now it follows from the recursion equations~\eqref{eq:global:positivity:too} and~\eqref{eq:global:positivity:ttoo} and the assumption $\upr(\xsing{E})=\umess{\init}>0$ that $\uamess{e}(\xval{e})>0$ for all $e\in E$. Again applying Eq.~\eqref{eq:global:positivity:too}, we find that indeed $\umess{c}(\xval{e})>0$ for all $c\in\children{e}$.
\end{proof}

Our proof of Theorem~\ref{theo:global-models} relies heavily on a very convenient coherence result proved by Enrique Miranda \cite[Theorem~6]{miranda2009a}, which we relate here to make the paper more self-contained. 
We use the notations established in the context of Section~\ref{sec:independent-natural-extension}.

\begin{theorem}\label{theo:blessyouquique}
  Let $\lpr$ be a (separately) coherent lower prevision on $\gambles{N}$, and consider $m$ disjoint pairs of subsets $O_k$ and $I_k$ of $N$, $k=1,\dots,m$. 
Assume that $\upr(\xsing{I_k})>0$ for all $\xval{I_k}\in\values{I_k}$, $k=1,\dots,m$ and use regular extension to define the conditional lower previsions $\lrext{I_k}$ on $\gambles{O_k}$, for $k=1,\dots,m$. 
Then the (conditional) lower previsions $\lpr$, $\lrext{I_1}$, \dots, $\lrext{I_m}$ are (jointly) coherent.
\end{theorem}

\begin{proof}[Proof of Theorem~\ref{theo:global-models}]
  We begin by showing that the family of models $\treefamily{P}$ satisfies requirements~T\ref{item:global:local}--T\ref{item:global:irrelevance}.
\par
  To prove~T\ref{item:global:local}, consider any $s\in\nodes$, and any $f\in\gambles{s}$, then it follows from separate coherence that $\linecon[f]{s}=f$, and therefore we infer from the recursion equation~\eqref{eq:backwards-global} that indeed $\lglobconm[f]{s}=\llocconm[{\linecon[f]{s}}]{s}=\llocconm[f]{s}$.
  \par
  Next, we turn to the proof of~T\ref{item:global:coherence} and~T\ref{item:global:irrelevance}.
  Consider any $s\in\nonterminals$, $S\subseteq\children{s}$ and $R\subseteq\nonparnondes{S}$.
  Let $\xval{\sing{s}\cup R}\in\values{\sing{s}\cup R}$ and $f\in\gambles{\after{S}\cup\sing{s}\cup R}$. 
  We calculate the following regular extension of the joint:
  \begin{equation}\label{eq:coherence-rext-new-one}
    \xinlrext[f]{\sing{s}\cup R}
    =\max\set{\mu\in\reals}{\lpr(\xindic{\sing{s}\cup R}[f-\mu])\geq0}.
  \end{equation}
  Consider that
  \begin{equation*}
    \xindic{\sing{s}\cup R}[f-\mu]=\xindic{s}\xindic{R}[g-\mu],
  \end{equation*}
  where $g\coloneqq f(\cdot,\xval{s},\xval{R})\in\gambles{\after{S}}$.
  Let $t_2$ be the unique child of $t_1\coloneqq\init$ such that $s\in\after{t_2}$ [assuming of course that $s\neq t_1=\init$].   
  By using separate coherence, recursion equations~\eqref{eq:backwards-nine}, \eqref{eq:backwards-global} and~\eqref{eq:backwards-restrict}, and the strong factorisation property [see Proposition~\ref{prop:independent-factorising}] of the conditionally independent natural extension, in a way similar to the argumentation in Section~\ref{sec:calculating-rho}, we see that
  \begin{align*}
    \lpr(\xindic{\sing{s}\cup R}[f-\mu])
    &=\lloc{t_1}(\lglobcon[{\xindic{R\setminus\after{t_2}}\lglobcon[{\xindic{s}\xindic{R\cap\after{t_2}}[g-\mu]}]
      {\children{t_1}}{t_1}}]{\children{t_1}}{t_1})\\
    &=\lloc{t_1}(\lglobcon[{h_2\lglobcon[{\xindic{s}\xindic{R\cap\after{t_2}}[g-\mu]}]{t_2}{t_1}}]
    {\children{t_1}}{t_1})\\
    &=\lglob{\after{t_1}}(h_2\lglobcon[{\xindic{s}\xindic{R\cap\after{t_2}}[g-\mu]}]{t_2}{t_1}),
\end{align*}
  where $h_2\coloneqq\xindic{R\setminus\after{t_2}}\geq0$.
  Similarly, let  $t_3$ be the unique child of $t_2$ such that $s\in\after{t_3}$ [assuming of course that $s\neq t_2$].
  Then we see in the same way as above that
  \begin{equation*}
    \lglobcon[{\xindic{s}\xindic{R\cap\after{t_2}}[g-\mu]}]{t_2}{t_1}
   =\lglobcon[{h_3\lglobcon[{\xindic{s}\xindic{R\cap\after{t_3}}[g-\mu]}]{t_3}{t_2}}]{t_2}{t_1},
  \end{equation*}
  where $h_3\coloneqq\xindic{R\setminus\after{t_3}}\geq0$.
  Continuing in this way, we eventually come to the conclusion that
  \begin{equation}\label{eq:coherence-new-one}
    \lpr(\xindic{\sing{s}\cup R}[f-\mu])
    =\underline{G}(\linecon[{h\xindic{s}[g-\mu]}]{s}),
  \end{equation}
  where $h\coloneqq\xindic{R\cap\after{\children{s}}}=\xindic{R\cap(\after{\children{s}}\setminus\after{S})}$,
  and where the real functional $\underline{G}$ on $\gambles{s}$ is essentially constructed as follows.
  Consider the segment $t_1t_2\dots t_{r-1}t_r$ connecting $\init$ and $s$, i.e.~$t_r\coloneqq s$, $t_{r-1}\coloneqq\mother{t_r}$, \dots, $t_{k}\coloneqq\mother{t_{k+1}}$, \dots, $t_1\coloneqq\mother{t_{2}}=\init$.
  Then there are non-negative $h_k$ on $\values{\after{t_k}}$ such that for all $f\in\gambles{s}$, $\underline{G}(f)=f_1$, where $f_r\coloneqq f$ and $f_{k}\coloneqq\lglobcon[{h_{k+1}f_{k+1}}]{t_{k+1}}{t_k}$, $k=1,\dots,r-1$. [If $r=1$, or in other words $s=t_1=\init$, just let $\underline{G}\coloneqq\lglob{\after{\init}}$.]
  In other words, the functional $\underline{G}$ results from recursively multiplying with non-negative maps and applying global conditional lower previsions. 
  As such, $\underline{G}$ is non-negatively homogeneous and super-additive [because the successive multiplication and composition preserves these properties].
  In addition, it does not depend on $g$ nor $\mu$.
 If we use the separate coherence of $\lglobcon{\children{s}}{s}$, the strong factorisation, associativity and marginalisation properties of the conditionally independent natural extension $\lglobcon{\children{s}}{s}$ [see Proposition~\ref{prop:independent-factorising}, Eqs.~\eqref{eq:marginalisation} and~\eqref{eq:associativity}, and the recursion equations~\eqref{eq:backwards-nine} and~\eqref{eq:backwards-restrict}], and the separate coherence of the conditional lower prevision $\xinlglobcon{S}{s}$, we get:
  \begin{align}
    \lglobcon[{h\xindic{s}[g-\mu]}]{\children{s}}{s}
    &=\xindic{s}\xinlglobcon[{h[g-\mu]}]{\children{s}}{s}\notag\\
    &=\xindic{s}
    \begin{cases}
      \xinlinecon[h]{s}[\xinlglobcon[g]{S}{s}-\mu]
      &\text{if $\xinlglobcon[g]{S}{s}\geq\mu$}\\
      \xinuinecon[h]{s}[\xinlglobcon[g]{S}{s}-\mu]
      &\text{if $\xinlglobcon[g]{S}{s}\leq\mu$}.
    \end{cases}
    \label{eq:coherence-new-two}
  \end{align}
  Combining Eqs.~\eqref{eq:coherence-new-one} and~\eqref{eq:coherence-new-two}, and invoking the non-negative homogeneity of the real functional $\underline{G}$, this leads to:
  \begin{equation*}
    \lpr(\xindic{\sing{s}\cup R}[f-\mu])
     =
     \begin{cases}
       \underline{G}(\xindic{s})\xinlinecon[h]{s}[\xinlglobcon[g]{S}{s}-\mu]
       &\text{if $\xinlglobcon[g]{S}{s}\geq\mu$}\\
       \overline{G}(\xindic{s})\xinuinecon[h]{s}[\xinlglobcon[g]{S}{s}-\mu]
       &\text{if $\xinlglobcon[g]{S}{s}\leq\mu$},
     \end{cases}
  \end{equation*}
  where we let $\overline{G}(\xindic{s})\coloneqq-\underline{G}(-\xindic{s})$.
  By letting $f=\mu\pm1$ [and therefore also $g=\mu\pm1$] in this expression, we derive in particular that
  \begin{equation*}
    \lpr(\xsing{\sing{s}\cup R})
    =\underline{G}(\xindic{s})\xinlinecon[h]{s}
    \text{ and }
    \upr(\xsing{\sing{s}\cup R})
    =\overline{G}(\xindic{s})\xinuinecon[h]{s},
  \end{equation*}
  and therefore also
  \begin{equation*}
    \lpr(\xindic{\sing{s}\cup R}[f-\mu])\\
     =
     \begin{cases}
       \lpr(\xsing{\sing{s}\cup R})[\xinlglobcon[g]{S}{s}-\mu]
       &\text{if $\xinlglobcon[g]{S}{s}\geq\mu$}\\
       \upr(\xsing{\sing{s}\cup R})[\xinlglobcon[g]{S}{s}-\mu]
       &\text{if $\xinlglobcon[g]{S}{s}\leq\mu$}.
     \end{cases}
   \end{equation*}
   Since we have assumed that all local models $\ulocconm{s}$ are strictly positive, we gather from Proposition~\ref{prop:global:positivity} that $\upr(\xsing{\sing{s}\cup R})>0$, and therefore
  \begin{equation*}
    \lpr(\xindic{\sing{s}\cup R}[f-\mu])\geq0
    \ifonlyif\xinlglobcon[g]{S}{s}\geq\mu.
  \end{equation*}
  This allows us to infer from Eq.~\eqref{eq:coherence-rext-new-one} that 
  \begin{equation}\label{eq:coherence-rext-new-two}
    \xinlrext[f]{\sing{s}\cup R}
    =\xinlglobcon[f(\cdot,\xval{s},\xval{R})]{S}{s}
    \text{ for all $f\in\gambles{\after{S}\cup\sing{s}\cup R}$ and $\xval{\sing{s}\cup R}\in\values{\sing{s}\cup R}$}.
  \end{equation}
  If we now combine Eq.~\eqref{eq:coherence-rext-new-two} with  Theorem~\ref{theo:blessyouquique}, we see that both T\ref{item:global:coherence} and T\ref{item:global:irrelevance} hold.
  \par
  To complete the proof, we consider T\ref{item:global:smallest}. 
  Consider any family of models $\treefamily{V}$ that satisfies conditions~T\ref{item:global:local}--T\ref{item:global:irrelevance}.
  Then we want to show that
  \begin{equation}\label{eq:dominance-one}
    \dlglobcon{S}{t}\geq\lglobcon{S}{t}
    \text{ for all $t\in\nonterminals$ and all non-empty $S\subseteq\children{s}$}
  \end{equation}
  and
  \begin{equation}\label{eq:dominance-two}
    \dlglob{}\geq\lglob{}.
  \end{equation}
  The proof proceeds in a recursive (inductive) fashion. 
  Since the $\dlglobconm{t}$ satisfy T\ref{item:global:local}, we infer in particular that
  \begin{equation*}
    \dlglobconm{t}=\lglobconm{t}=\llocconm{t}
    \text{ for all terminal nodes $t$}.
  \end{equation*}
  It is therefore clearly sufficient to prove the following statement for all non-terminal nodes $t\in\nonterminals$:
  \begin{equation}\label{eq:induction-step}
    (\forall c\in\children{t})
    (\dlglobcon{c}{t}\geq\lglobcon{c}{t})
    \then
    \left\{
      \begin{aligned}
        &\dlglobcon{S}{t}\geq\lglobcon{S}{t}
        \text{ for all non-empty $S\subseteq\children{t}$}\\
        &\dlglobconm{t}\geq\lglobconm{t}.
      \end{aligned}
    \right.
  \end{equation}
  This is what we now set out to do.
  Fix any non-terminal node $t\in\nonterminals$ and any non-empty $S\subseteq\children{t}$, and assume that $\dlglobcon{c}{t}\geq\lglobcon{c}{t}$ for all $c\in\children{t}$.
  \par
  First of all, define for any disjoint proper subsets $I$ and $O$ of $S$, the conditional lower previsions $\dlglobcon{O}{\sing{t}\cup\after{I}}$ through:
  \begin{equation*}
    \xindlglobcon[f]{O}{\sing{t}\cup\after{I}}
    =\xindlglobcon[f(\cdot,\xval{\after{I}})]{O}{t}
    \text{ for all $f\in\gambles{\after{O}\cup\after{I}}$ and all $\xval{\sing{t}\cup\after{I}}\in\values{\sing{t}\cup\after{I}}$}.
  \end{equation*}
  Then we infer from T\ref{item:global:irrelevance} [with $S=O$, $s=t$ and $R=\after{I}\subseteq\nonparnondes{O}$] that all these conditional lower previsions are in particular (jointly) coherent with the conditional lower prevision $\dlglobcon{S}{t}$. 
  If we recall Definition~\ref{def:conditionally:independent:product} [with $N=\set{\after{c}}{c\in S}$ and $\vartoo=\var{t}$], we conclude that  $\dlinecon{t}$ is a conditionally independent product of the `marginals' $\dlglobcon{c}{t}$, $c\in S$, which therefore dominates the smallest such independent product:
  \begin{equation*}
    \dlglobcon{S}{t}
    \geq\otimes_{c\in S}\dlglobcon{c}{t}
  \end{equation*}
  and therefore, using the assumption, we infer from this inequality that
  \begin{equation}\label{eq:other-step-two}
    \dlglobcon{S}{t}
    \geq\otimes_{c\in S}\dlglobcon{c}{t}
    \geq\otimes_{c\in S}\lglobcon{c}{t}
    =\lglobcon{S}{t},
  \end{equation}
  where we have also used, successively, the monotonicity property of the conditionally independent natural extension [see \cite{cooman2009b} for a proof] and the recursion equations~\eqref{eq:backwards-nine} and~\eqref{eq:backwards-nine-restrict}.
  \par
  Next, define the conditional lower prevision $\dlpr_{\after{\children{t}}}(\cdot\vert\var{\sing{\mother{t},t}})$ on $\gambles{\msing{t}\cup\after{t}}$ through:
  \begin{multline}\label{eq:other-step-three}
    \dlpr_{\after{\children{t}}}(f\vert\xval{\sing{\mother{t},t}})
    \coloneqq
    \xindlinecon[f(\xval{\mother{t}},\xval{t},\cdot)]{t}\\
    \text{for all $f\in\gambles{\msing{t}\cup\after{t}}$ and all $\xval{\sing{\mother{t},t}}\in\values{\sing{\mother{t},t}}$}.  
  \end{multline}
  Then we infer from T\ref{item:global:irrelevance} [with $s=t$, $S=\children{t}$ and $R=\msing{t}\subseteq\nonparnondes{\children{t}}$] that this conditional lower prevision $\dlpr_{\after{\children{t}}}(\cdot\vert\var{\sing{\mother{t},t}})$ is in particular (jointly) coherent with the conditional lower prevision $\dlglobconm{t}$ defined on $\gambles{\msing{t}\cup\after{t}}$. 
  We then see that for all $g\in\gambles{\after{t}}$:
 \begin{align*}
    \dlglobconm[g]{t}
    &\geq\dlglobconm[{\dlpr_{\after{\children{t}}}(g\vert\var{\sing{\mother{t},t}})}]{t}\\
    &=\dlglobconm[{\dlinecon[g]{t}}]{t}
    =\llocconm[{\dlinecon[g]{t}}]{t}\\
    &\geq\llocconm[{\linecon[g]{t}}]{t}\\
    &=\lglobconm{t}.
  \end{align*}
  The first equality follows from Eq.~\eqref{eq:other-step-three}, the second one holds because the global models $\dlglobconm{t}$ satisfy T\ref{item:global:local}, and the third one follows from recursion equation~\eqref{eq:backwards-global}. 
  The first inequality follows if we apply Walley's Marginal Extension Theorem\footnote{Recall that this is a coherence result that generalises the so-called Law of Iterated Expectations to coherent lower previsions.} \cite[Theorem~6.7.2]{walley1991} in the formulation of \cite[Theorem~4]{miranda2006b}. 
  The second inequality follows from the inequality~\eqref{eq:other-step-two} and the non-decreasing character of $\llocconm{t}$, which follows from separate coherence.
  This completes our proof that T\ref{item:global:smallest} is also satisfied.
  \par
  The last part of the proof follows at once from Eqs.~\eqref{eq:coherence-rext-new-two} [with $R=\emptyset$], and Theorem~\ref{theo:blessyouquique}.
\end{proof}
\end{document}